\documentclass[a4paper,letter]{article}

\usepackage{amsmath}
\usepackage{amsthm}
\usepackage{graphicx}
\usepackage{times}
\usepackage{subcaption}

\usepackage{natbib}
\setcitestyle{authoryear,round,citesep={;},aysep={,},yysep={;}}

\graphicspath{{../}}

\usepackage[margin=1in]{geometry}

\usepackage{amsmath,amsfonts,bm}

\def\eqref#1{equation~\ref{#1}}
\def\1{\bm{1}}

\DeclareMathAlphabet{\mathsfit}{\encodingdefault}{\sfdefault}{m}{sl}
\SetMathAlphabet{\mathsfit}{bold}{\encodingdefault}{\sfdefault}{bx}{n}

\newcommand{\E}{\mathbb{E}}

\newcommand{\R}{\mathbb{R}}

\usepackage{hyperref}
\usepackage{url}

\def\sgn{\operatorname{sign}}
\def\Int{\operatorname{Int}}
\def\Bd{\operatorname{Bd}}
\def\rowspace{\operatorname{rowspace}}
\def\nullspace{\operatorname{nullspace}}
\newtheorem{theorem}{Theorem} % section
\newtheorem{lemma}[theorem]{Lemma}
\newtheorem{cor}[theorem]{Corollary}

\title{Adaptive versus Standard Descent Methods and\\ Robustness Against Adversarial Examples}

\author{Marc Khoury \footnote{khoury@eecs.berkeley.edu}\\
University of California, Berkeley
}

\begin{document}

\maketitle

\begin{abstract}
Adversarial examples are a pervasive phenomenon of machine learning models where seemingly imperceptible perturbations to the input lead to misclassifications for otherwise statistically accurate models. In this paper we study how the choice of optimization algorithm influences the robustness of the resulting classifier to adversarial examples. Specifically we show an example of a learning problem for which the solution found  by adaptive optimization algorithms exhibits qualitatively worse robustness properties against both $L_{2}$- and $L_{\infty}$-adversaries than the solution found by non-adaptive algorithms. Then we fully characterize the geometry of the loss landscape of $L_{2}$-adversarial training in least-squares linear regression. The geometry of the loss landscape is subtle and has important consequences for optimization algorithms. Finally we provide experimental evidence which suggests that non-adaptive methods consistently produce more robust models than adaptive methods.

\textbf{keywords:} adversarial examples, robustness, optimization, geometry

\end{abstract}

\thispagestyle{empty}
\setcounter{page}{0}
\newpage

\section{Introduction}
\label{sec:intro}
Adversarial examples are a pervasive phenomenon of machine learning models where perturbations of the input that are imperceptible to humans reliably lead to confident incorrect classifications (\cite{Szegedy13,Goodfellow14}). Since this phenomenon was first observed, researchers have attempted to develop methods which produce models that are robust to adversarial perturbations under specific attack models (\cite{Wong18a, Sinha18, Raghunathan18, Mirman18, Madry17, Zhang19}). As machine learning proliferates into society, including security-critical settings like health care (\cite{Esteva17}) or autonomous vehicles (\cite{Codevilla18}), it is crucial to develop methods that allow us to understand the vulnerability of our models and design appropriate counter-measures.

Additionally there is a growing literature on the theory of adversarial examples. Many of these results attempt to understand adversarial examples by constructing examples of learning problems for which it is difficult to construct a classifier that is robust to adversarial perturbations. This difficulty may arise due to sample complexity (\cite{Schmidt18}), computational constraints (\cite{Bubeck19, Degwekar19}), or the high-dimensional geometry of the initial feature space (\cite{Shafahi19, Khoury18}). We expand upon these results in Section~\ref{sec:related}.

Currently less well-understood, and to our knowledge not addressed by the theoretical literature on adversarial examples, is how our algorithmic choices effect the robustness of our models. With respect to optimization and generalization, but importantly not robustness, the success of standard (or \emph{non-adaptive}) gradient descent methods, including stochastic gradient descent (SGD) and SGD with momentum, is starting to be better understood (\cite{Du19, AllenZhu19, Gunasekar18a, Gunasekar18b}). However, as an increasing amount of time has been spent training deep networks, researchers and practitioners have heavily employed \emph{adaptive} gradient methods, such as Adam (\cite{Kingma15}), Adagrad (\cite{Duchi11}), and RMSprop (\cite{Tieleman12}), due to their rapid training times (\cite{Karparthy17}). Unfortunately the properties of adaptive optimization algorithms are less well-understood than those of their non-adaptive counterparts. \cite{Wilson17} provide theoretical and empirical evidence which suggests that adaptive algorithms often produce solutions that generalize worse than those found by non-adaptive algorithms. 

In this paper, we study the robustness of solutions found by adaptive and non-adaptive algorithms to adversarial examples. Furthermore we study the effect of adversarial training on the geometry of the loss landscape and, consequently, on the solutions found by adaptive and non-adaptive algorithms for the adversarial training objective. Our paper makes the following contributions. 

\begin{itemize}
    \item We show an example of a learning problem for which the solution found by adaptive optimization algorithms exhibits qualitatively worse robustness properties against \emph{both} $L_{2}$- and $L_{\infty}$-adversaries than the solution found by non-adaptive algorithms. Furthermore the robustness of the adaptive solution decreases rapidly as the dimension of the problem increases, while the robustness of the non-adaptive solution is stable as the dimension increases. 
    \item We fully characterize the geometry of the loss landscape of $L_{2}$-adversarial training in least-squares linear regression. The $L_{2}$-adversarial training objective $\mathcal{L}_{2}$ is convex everywhere; moreover, it is strictly convex everywhere except along either 0, 1, or 2 line segments, depending on the value of $\epsilon$. Furthermore for nearly all choices\footnote{For all $\epsilon \neq 1/ \|X^{\dagger}y\|_{2}$} of $\epsilon$, these line segments along which $\mathcal{L}_{2}$ is convex, but not strictly convex, lie outside of the rowspace and the gradient along these line segments is nonzero. It follows that any reasonable optimization algorithm finds the unique global minimum of $\mathcal{L}_{2}$.  
    \item We conduct an extensive empirical evaluation to explore the effect of different optimization algorithms on robustness. Our experimental results suggest that non-adaptive methods consistently produce more robust models than adaptive methods. 
    \item We provide a dataset consisting of 190 pretrained models on MNIST and CIFAR10 with various hyperparameter settings. Of these 190 pretrained models, 150 were used to find the best hyperparamter settings for our experiments and evaluated on a validation set. The remaining 40 pretrained models were evaluated on the test set. Of the 150 validation models, 88 were trained using natural training and 62 were trained using adversarial training. Of the 40 test models, 20 were trained using natural training and 20 were trained using adversarial training. They can be downloaded at \href{https://www.dropbox.com/s/edfcnb97lzxl19z/models.zip}{https://www.dropbox.com/s/edfcnb97lzxl19z/models.zip}.
\end{itemize}

\section{Related Work}
\label{sec:related}

There has been a long line of work on the theory of adversarial examples. \cite{Schmidt18} explore the sample complexity required to produce robust models. They demonstrate a simple setting, a mixture of two Gaussians, in which a linear classifier with near perfect natural accuracy can be learned from a single sample, but \emph{any} algorithm that produces \emph{any} binary classifier requires $\Omega(\sqrt{d})$ samples to produce a robust classifier. Followup work by \cite{Bubeck19} suggests that adversarial examples may arise from computational constraints. They exhibit pairs of distributions that differ only in a $k$-dimensional subspace, and are otherwise standard Gaussians, and show that while it is information-theoretically possible to distinguish these distributions, it requires exponentially many queries in the statistical query model of computation. We note that both of these constructions produce distributions whose support is the entirety of $\R^d$. 

\cite{Bubeck19} further characterize five mutually exclusive ``worlds'' of robustness, inspired by similar characterizations in complexity theory (\cite{Impagliazzo95}). A learning problem must fall into one of the following possibilities: 
\begin{enumerate}
    \item[]\textbf{World 1}: No robust classifier exists, regardless of computational considerations or sample efficiency.
    \item[]\textbf{World 2}: Robust classifiers exists, but they are computationally inefficient to evaluate.
    \item[]\textbf{World 3}: Computationally efficient robust classifiers exist, but learning them requires more samples.
    \item[]\textbf{World 4}: Computationally efficient robust classifiers exist and can be learned from few samples, but learning is inefficient.
    \item[]\textbf{World 5}: Computationally efficient robust classifiers exists and can be learned efficiently from few samples.
\end{enumerate}

While learning problems can be constructed that fall into each possible world, the question for researchers is into which world are problems from practice most likely to fall? Every theoretical construction, such as those by \cite{Schmidt18} and \cite{Bubeck19}, can be thought of as providing evidence for the prevalence of one of the worlds. In the language of \cite{Bubeck19}, the sampling complexity result of \cite{Schmidt18} provides evidence for World 3, by constructing an example of a problem that falls into world three. The learning problem constructed by \cite{Bubeck19} provides evidence for World 4. Subsequent work by \cite{Degwekar19} provides evidence for Worlds 2 and 4. Under standard cryptographic assumptions, \cite{Degwekar19}  construct an a learning problem for which a computationally efficient non-robust classifier exists, no efficient robust classifier exists, but an inefficient robust classifier exists. Similarly, assuming the existence of one-way functions, they construct a learning problem for which an efficient robust classifier exists, but it is computationally inefficient to learn a robust classifier. Finally, in an attempt to understand how likely World 4 is in practice, they show that any task where an efficient robust classifier exists but is hard to learn in polynomial time implies one-way functions.\footnote{Thus at least one community will be happy.} 

Additionally there is a line work that attempts to explain the pervasiveness of adversarial examples through the lens of high-dimensional geometry. \cite{Gilmer18} experimentally evaluated the setting of two concentric under-sampled $499$-spheres embedded in $\R^{500}$, and concluded that adversarial examples occur on the data manifold. \cite{Shafahi19} suggest that adversarial examples may be an unavoidable consequence of the high-dimensional geometry of data. Their result depends upon the use of an isopermetric inequality. The main drawback of these results, as well as the constructions of \cite{Schmidt18} and \cite{Bubeck19}, is that they assume that the support of the data distribution has full or nearly full dimension. We do not believe this to be the case in practice. Instead we believe that the data distribution is often supported on a very low-dimensional subset of $\R^d$. This case is addressed in \cite{Khoury18}, where they consider the problem of constructing decision boundaries robust to adversarial examples when data is drawn from a low-dimensional manifold embedded in $\R^d$. They highlight the role of co-dimension, the difference between the dimension of the embedding space and the dimension of the data manifold, as a key source of the pervasiveness of adversarial vulnerability. Said differently, it is the low-dimensional structure of features embedded in high-dimensional space that contributes, at least in part, to adversarial examples. This idea is also explored in \cite{Nar19}, but with emphasis on the cross-entropy loss. 

We believe that problems in practice are most likely to fall into World 5, the best of all worlds. Problems in this class have robust classifiers which are efficient to evaluate and can be learned efficiently from relatively few samples. We simply haven't found the right algorithm for learning such classifiers. The goal of this paper is to explore the effect of our algorithms on robustness. Specifically we wish to understand the robustness properties of solutions found by common optimization algorithms. To our knowledge no other work has explored the robustness properties of solutions found by different optimization algorithms. 

\section{Adaptive Algorithms May Significantly Reduce Robustness}
\label{sec:adpvsgd}
\cite{Wilson17} explore the effect of different optimization methods on generalization both in a simple theoretical setting and empirically. For their main theoretical result, they construct a learning problem for which the solution found by \emph{any} adaptive method, denoted $w_{\operatorname{ada}}$, has worse generalization properties than the solution found by non-adaptive methods, denoted $w_{\text{SGD}}$. We recall their construction in Section~\ref{sec:simpleprob}. In Section~\ref{sec:adasol} we describe the adaptive solution $w_{\operatorname{ada}}$ and in Section~\ref{sec:sgdsol} we describe the non-adaptive solution $w_{\text{SGD}}$. 

Generalization and robustness are different properties of a classifier. A classifier can generalize well but have terrible robustness properties, as we often see in practice. On the other hand, a constant classifier generalizes poorly, but has perfect robustness (\cite{Zhang19}). \cite{Wilson17} study the \emph{generalization} properties of $w_{\operatorname{ada}}$ and $w_{\text{SGD}}$, but not their robustness properties. In Section~\ref{sec:optrobustresults} we study the robustness properties of $w_{\operatorname{ada}}$ and $w_{\text{SGD}}$. Specifically, we show that $w_{\text{SGD}}$ exhibits superior robustness properties to $w_{\operatorname{ada}}$ against \emph{both} $L_{2}$- and $L_{\infty}$-adversaries.

\subsection{A Simple Learning Problem}
\label{sec:simpleprob}

Let $X \in \R^{n \times d}$ be a design matrix representing a dataset with $n$ sample points and $d$ features and let $y \in \{\pm 1\}^{n}$ be a vector of labels. \cite{Wilson17} restrict their attention to binary classification problems of this type, and learn a classifier by minimizing the least-squares loss
\begin{equation}
\label{equ:leastsquares}
   \min_{w} \mathcal{L}(X, y; w) = \min_{w} \frac{1}{2} \|Xw -y \|_{2}^2.
\end{equation}

They construct the following learning problem for which they can solve for both the adaptive and non-adaptive solutions in closed form. Their construction uses an infinite-dimensional feature space for simplicity, but they note that $6n$ dimensions suffice. For $i \in 1 \ldots n$, sample $y_i = 1$ with probability $p$, and $y_i=-1$ with probability $1 -p$ for some $p > 0.5$. Then set $x_i$ to be the infinite-dimensional vector 

\begin{equation}
\label{equ:datasetdef}
    x_{ij} = \begin{cases} 
      y_i & j = 1 \\
      1 & j = 2,3 \\
      1 & j = 4 + 5(i-1) \\
      (1-y_i)/2 & j = 5 + 5(i-1), \ldots, 8+5(i-1)\\
      0 & \text{otherwise}.
   \end{cases}
\end{equation}

For example, a dataset with three sample points following Equation~\ref{equ:datasetexp} is
\setcounter{MaxMatrixCols}{20}
\begin{equation}
\label{equ:datasetexp}
    \begin{pmatrix}
    1 & 1 & 1 & 1 & 0 & 0 & 0 & 0 & 0 & 0 & 0 & 0 & 0 & 0 & 0 & 0 & 0 & 0 & 0 & \ldots\\
    -1 & 1 & 1 & 0 & 0 & 0 & 0 & 0 & 1 & 1 & 1 & 1 & 1 & 0 & 0 & 0 & 0 & 0 & 0 & \ldots\\
    1 & 1 & 1 & 0 & 0 & 0 & 0 & 0 & 0 & 0 & 0 & 0 & 0 & 1 & 0 & 0 & 0 & 0 & 0 & \ldots
    \end{pmatrix}.
\end{equation}

The first feature encodes the label, and is alone sufficient for classification. Note that this trick of encoding the label is also commonly used in the robustness literature to construct examples of hard-to-learn-robustly problems (\cite{Bubeck19,Degwekar19}). The second and third feature are identically $1$ for every sample. Then there is a subset of five dimensions which are identified with $x_i$ and contain a set of features which are \emph{unique} to $x_i$. If $y_i = 1$ then there is a single $1$ in this subset of five dimensions and $x_i$ is the only sample with a $1$ in this dimension. If $y_i = -1$ then all five dimensions are set to $1$ and again $x_i$ is the only sample with a $1$ at these five positions. 

While this problem may seem contrived, it contains several properties that are common in machine learning problems and that are particularly important for robustness. It contains a single robust feature that is strongly correlated with the label. However it may not be easy for an optimization algorithm to identify such a feature. Additionally there are many non-robust features which are weakly or not at all correlated with the label, but which may appear useful for generalization because they are uniquely identified with samples from specific classes. \cite{Wilson17} show that both adaptive and non-adaptive methods find classifiers that place at least some weight on every nonzero feature.

\subsection{The Adaptive Solution $w_{\operatorname{ada}}$}
\label{sec:adasol}
Let $(X, y)$ be generated by the generative model in Section~\ref{sec:simpleprob}. When initialized at the origin, \cite{Wilson17} show that \emph{any} adaptive optimization algorithm -- such as RMSprop, Adam, and Adagrad -- minimizing Equation~\ref{equ:leastsquares} for $(X,y)$ converges to $w_{\operatorname{ada}} \propto v$ where

\begin{equation}
%u = \begin{cases} 
%      n & j = 1 \\
%      b & j = 2,3 \\
%      y_{\lfloor \frac{j+1}{5} \rfloor} & j > 3 \text{ and } x_{\lfloor \frac{j+1}{5} \rfloor} = 1\\
%      0 & \text{otherwise}
%      \end{cases}, \quad \text{and} \quad
    v = \begin{cases} 
      1 & j = 1 \\
      1 & j = 2,3 \\
      y_{\lfloor (j+1)/5 \rfloor} & j > 3 \text{ and } x_{\lfloor (j+1)/5 \rfloor} = 1\\
      0 & \text{otherwise}.
      \end{cases}
\end{equation}

Thus we can write $w_{\operatorname{ada}} = \tau v$ for some positive constant $\tau > 0$. On a test example $(x_{\operatorname{test}}, y_{\operatorname{test}})$, that is distinct from all the training examples, $\langle w_{\operatorname{ada}},x_{\operatorname{test}} \rangle = \tau(y_{\operatorname{test}} + 2) > 0$. Thus $w_{\operatorname{ada}}$ labels every unseen example as a positive example.

\subsection{The Non-adaptive Solution $w_{\operatorname{SGD}}$}
\label{sec:sgdsol}
For $(X, y)$, let $\mathcal{P}, \mathcal{N}$ denote the sets of positive and negative samples in $X$ respectively. Let $n_{+} = |\mathcal{P}|, n_{-} = |\mathcal{N}|$ and note that $n = n_{+} + n_{-}$. When the weight vector is initialized in the row space of $X$, \cite{Wilson17} show that all non-adaptive methods -- such as gradient descent, SGD, SGD with momentum, Nesterov's method, and conjugate gradient -- converge to $w_{\text{SGD}} = X^{\dagger}y$, where $X^{\dagger}$ denotes the pseudo-inverse. That is, among the infinitely many solutions of the underdetermined system $Xw = y$, non-adaptive methods converge to the solution which minimizes $\|w\|_{2}$, and thus maximizes the $L_{2}$-margin. Specifically $w_{\text{SGD}} = \sum_{i \in \mathcal{P}} \alpha_{+} x_i + \sum_{j \in \mathcal{N}} \alpha_{-} x_j$ where 

\begin{equation*}
    \alpha_{+} = \frac{4n_{-} + 5}{15n_{+} + 3n_{-}+8n_{+}n_{-}+5},\quad \alpha_{-} =- \frac{4n_{+} + 1}{15n_{+} + 3n_{-}+8n_{+}n_{-}+5}.
\end{equation*}

Note that these values for $\alpha_{+}, \alpha_{-}$ differ slightly from those presented in \cite{Wilson17}. In Appendix~\ref{sec:corrections} we discuss in detail two errors in their derivation that lead to this discrepancy. These errors do not qualitatively change their results. Furthermore, in Appendix~\ref{sec:sgdrobustproof} we carefully discuss under what conditions $\langle w_{\text{SGD}}, x_{\operatorname{test}}\rangle$ is positive and negative for $y_{\operatorname{test}} = \pm 1$. For now, we simply state that for all $n_{+},n_{-} \geq 1$, $w_{\text{SGD}}$ correctly classifies every test example.

\subsection{Analyzing the Robustness of $w_{\operatorname{ada}}$ and $w_{\operatorname{SGD}}$}
\label{sec:optrobustresults}
In this section we analyze the robustness properties of $w_{\operatorname{ada}}$ and $w_{\text{SGD}}$ against $L_{2}$- and $L_{\infty}$-adversaries. We show that $w_{\text{SGD}}$ exhibits considerably more robustness against \emph{both} $L_{2}$- and $L_{\infty}$-adversaries than $w_{\operatorname{ada}}$. A priori this is surprising; one may have expected $w_{\operatorname{ada}}$, which is a small $L_{\infty}$-norm solution, to be more robust to $L_{\infty}$-perturbations, while $w_{\text{SGD}}$, which is a small $L_{2}$-norm solution, to be robust to $L_{2}$-perturbations. However this expectation is wrong. Interestingly the robustness of $w_{\operatorname{ada}}$ against both $L_{2}$- and $L_{\infty}$-adversaries decreases as the dimension increases, whereas the robustness of $w_{\text{SGD}}$ does not. Finally, neither method recovers the ``obvious'' solution $w^{*} = (1, 0, \ldots, 0)$, which generalizes well and is optimally robust against both $L_2$- and $L_{\infty}$-perturbations. 

Theorems~\ref{thm:adaptiverobust} and \ref{thm:sgdrobust} are our main results of this section; the proofs are deferred to Appendix~\ref{sec:proofs}. We start by computing the robustness of $w_{\operatorname{ada}}$ against $L_{2}$- and $L_{\infty}$-adversaries. 

\begin{theorem}
\label{thm:adaptiverobust}
Let $(x_{\operatorname{test}}, y_{\operatorname{test}})$ be a test sample that is correctly classified by $w_{\operatorname{ada}}$ and let $\delta \in \R^d$ be a perturbation. The adaptive solution $w_{\operatorname{ada}}$ is robust against any $L_{2}$-perturbation for which
\begin{equation}
    \|\delta\|_{2} < \frac{\sqrt{9n_{+} + 1125n_{-} + 27}}{25n_{-}+n_{+} + 3}
\end{equation}
and any $L_{\infty}$-perturbation for which
\begin{equation}
    \|\delta\|_{\infty} < \frac{3}{3 + n_{+} + 5 n_{-}}.
\end{equation}

Furthermore these bounds are tight, meaning that an $L_{2}$- or $L_{\infty}$-ball with these radii centered at $x_{\operatorname{test}}$ intersects the decision boundary.
\end{theorem}

\begin{cor}
\label{cor:adaptiverobust}
Asymptotically, the $L_{2}$- and $L_{\infty}$-robustness of $w_{\operatorname{ada}}$ are, respectively, 
\begin{equation*}
   \Theta\left(\frac{1}{\sqrt{n_{+} + n_{-}}}\right) \text{ and } \Theta\left(\frac{1}{n_{+} + n_{-}}\right).
\end{equation*}
In particular both the $L_{2}$- and $L_{\infty}$-robustness go to $0$ as the number of samples $n_{+}, n_{-} \rightarrow \infty$.
\end{cor}

Corollary~\ref{cor:adaptiverobust} makes clear, qualitatively, the result in Theorem~\ref{thm:adaptiverobust}. The rate at which the $L_{2}$- and $L_{\infty}$-robustness of $w_{\operatorname{ada}}$ decrease reflects a dependence on dimension. The number of dimensions on which $w_{\operatorname{ada}}$ puts nonzero weight increases as we increase the number of samples, which reduces robustness. We also find it interesting that, despite classifying every test point as a positive example,   $w_{\operatorname{ada}}$'s predictions on correctly classified test samples are brittle. In summary, $w_{\operatorname{ada}}$ exhibits nearly no robustness against $L_2$- or $L_{\infty}$-adversaries.

Next we show that $w_{\text{SGD}}$ exhibits significant robustness against \emph{both} $L_{2}$- and $L_{\infty}$-adversaries.

\begin{theorem}
\label{thm:sgdrobust}
Let $(x_{\operatorname{test}}, y_{\operatorname{test}})$ be a test sample that is correctly classified by $w_{\text{SGD}}$ and let $\delta \in \R^d$ be a perturbation. The SGD solution $w_{\text{SGD}}$ is robust against any $L_{2}$-perturbations for which
\begin{equation}
    \|\delta\|_{2} \leq \begin{cases}
     \frac{15n_{+}+8n_{+}n_{-}-n_{-}}{\sqrt{64n_{+}^2n_{-}^2 + 160n_{+}^2n_{-} + 75n_{+}^2 + 32n_{+}n_{-}^2+60n_{+}n_{-}+70n_{+} + 3n_{-}^2 + 5n_{-}}}& y_{\operatorname{test}} = 1\\
    \frac{-5n_{+}+8n_{+}n_{-}+3n_{-}}{\sqrt{64n_{+}^2n_{-}^2 + 160n_{+}^2n_{-} + 75n_{+}^2 + 32n_{+}n_{-}^2+60n_{+}n_{-}+70n_{+} + 3n_{-}^2 + 5n_{-}}} & y_{\operatorname{test}} = -1
    \end{cases}
\end{equation}
and any $L_{\infty}$-perturbation for which 
\begin{equation}
    \|\delta\|_{\infty} \leq \begin{cases}
    \frac{15n_{+} +8n_{+}n_{-} - n_{-}}{20n_{+} + 32n_{+}n_{-} + 4n_{-}} & y_{\operatorname{test}} = 1\\
    \frac{-5n_{+} +8n_{+}n_{-} + 3n_{-}}{20n_{+} + 32n_{+}n_{-} + 4n_{-}} & y_{\operatorname{test}} = -1.
    \end{cases}
\end{equation}
Furthermore these bounds are tight, meaning that an $L_{2}$- or $L_{\infty}$-ball with these radii centered at $x_{\operatorname{test}}$ intersects the decision boundary.
\end{theorem}

\begin{cor}
Asymptotically, the $L_{2}$- and $L_{\infty}$-robustness of $w_{\text{SGD}}$ are both $\Theta\left(1\right)$. In particular the $L_{2}$-robustness approaches $1$ and the $L_{\infty}$-robustness approaches $\frac{1}{4}$ as the number of samples $n_{+}, n_{-} \rightarrow \infty$.
\end{cor}

Unsurprisingly, $w_{\text{SGD}}$, which maximizes the $L_2$-margin, exhibits near optimal robustness against $L_2$-adversaries. As the number of samples increases, the $L_2$-robustness of $w_{\text{SGD}}$ approaches $1$. Perhaps surprisingly, $w_{\text{SGD}}$ also exhibits moderate robustness to $L_{\infty}$-perturbations. As the number of samples increases, the $L_{\infty}$-robustness of $w_{\text{SGD}}$ approaches $\frac{1}{4}$. Unlike $w_{\operatorname{ada}}$, the amount of robustness exhibited by $w_{\text{SGD}}$ does not decrease as the dimension increases, instead asymptotically approaching a constant. 

However the $L_2$-robustness of $w_{\text{SGD}}$ is not exactly $1$ for any finite sample. One class $(y_{\operatorname{test}} = 1)$ approaches $1$ from above, while the other class $(y_{\operatorname{test}} = -1)$ approaches $1$ from below. To maximize the margin, $w_{\text{SGD}}$ places a small amount of weight on every other nonzero feature, even though all but the first are useless for classification. This lack of sparsity is also what causes the $L_{\infty}$-robustness to drop from a possible maximum of $1$ to $\frac{1}{4}$. In contrast, $w^{*} = (1, 0, \ldots, 0)$ generalizes perfectly, has $L_2$-robustness equal to $1$ for both classes, and, as an added benefit, has $L_{\infty}$-robustness equal to $1$ for both classes. Thus we have an example of a problem for which the max $L_{2}$-margin solution could reasonably be considered to not be the best classifier against $L_{2}$-perturbations. 

Furthermore, $w^{*}$ is \emph{not} in the row space of $X$. ($w_{\text{SGD}}$ is the projection of $w^{*}$ onto the row space.) Thus non-adaptive methods, when restricted to the row space, are \emph{incapable} of recovering $w^{*}$, irrespective of sample complexity (\cite{Schmidt18}) or computational considerations (\cite{Bubeck19}). This is simply the wrong algorithm for the desired objective. In the next section we study the effect that adversarial training has on the loss landscape and on the solutions found by various optimization algorithms. 

\section{Adversarial Training (Almost) Always Helps}
In the previous section we presented a learning problem for which adaptive optimization methods find a solution with significantly worse robustness properties against both $L_2$- and $L_{\infty}$-adversaries compared to non-adaptive methods. In this section we consider a different algorithm, adversarial training, for finding robust solutions to Equation~\ref{equ:leastsquares}. We are interested in two questions. First, does adversarial training sufficiently regularize the loss landscape so that adaptive  and non-adaptive methods find solutions with identical or qualitatively similar robustness properties? Second, are the solutions to the robust objective qualitatively different than those found by natural training or does adversarial training simply choose a robust solution from the space of solutions to the natural problem? We address the first question in Section~\ref{sec:lossgeometry} for $L_{2}$-adversarial training and the second in Section~\ref{sec:charsol} for the learning problem defined in Section~\ref{sec:adpvsgd}.

\subsection{The Adversarial Training Objective}
\cite{Madry17} formalize adversarial training by introducing the robust objective 
\begin{equation}
\min_{w} \E_{(x,y)\in \mathcal{D}}\left[\max_{\delta \in \Delta}  \mathcal{L}(x + \delta, y; w)\right]
\end{equation}
where $\mathcal{D}$ is the data distribution, $\Delta$ is a perturbation set meant to enforce a desired constraint, and $\mathcal{L}$ is a loss function. The goal then is to find a setting of the parameters $w$ of the model that minimize the expected loss against the worst-case perturbation in $\Delta$. 

Take $\mathcal{L}$ as in Equation~\ref{equ:leastsquares} and $\Delta$ to be an $L_{p}$-ball of radius \mbox{$\epsilon > 0$}. In the linear case, we can solve the inner maximization problem exactly.

\begin{align}
    \max_{\{\delta_{i}\}_{i \in [n]} \in \Delta^n} \mathcal{L}(x+\delta_i, y; w) 
    %&= \max_{\{\delta_{i}\}_{i \in [n]} \in \Delta^n} \frac{1}{2}\|(X+D)w - y\|_{2}^{2}\nonumber\\
\label{equ:objform1}    &= \max_{\{\delta_{i}\}_{i \in [n]} \in \Delta^n} \frac{1}{2} \sum_{i=1}^{n} \left(\langle x_i + \delta_i, w\rangle - y_i\right)^2\\
%\label{equ:objform1}    &=  \max_{\{\delta_{i}\}_{i \in [n]} \in \Delta^n} \frac{1}{2} \sum_{i=1}^{n} \left(\langle x_i, w\rangle - y_i + \langle \delta_i, w\rangle \right)^2\\
    &=  \max_{\{\delta_{i}\}_{i \in [n]} \in \Delta^n} \frac{1}{2} \sum_{i=1}^{n} \left(\left(\langle x_i, w\rangle - y_i\right)^2 + 2 \langle \delta_i, w \rangle (\langle x_i, w\rangle - y_i)  + \langle \delta_i, w\rangle^2 \right)\nonumber\\
    &= \frac{1}{2} \sum_{i=1}^{n} \left(\left(\langle x_i, w\rangle - y_i\right)^2 + 2 \epsilon \|w\|_{*}\sgn(\langle x_i, w\rangle - y_i) (\langle x_i, w\rangle - y_i)  + \epsilon^2\|w\|_{*}^2 \right)\nonumber\\
\label{equ:objform2}    &= \frac{1}{2}\|Xw - y\|_{2}^{2} + \epsilon\|w\|_{*}\|Xw - y\|_{1} + \frac{\epsilon^2n}{2}\|w\|_{*}^2.
\end{align}

The third identity follows from the definition of the dual norm, where $\|\cdot\|_{*}$ denotes the norm dual to the $L_p$ norm that defines $\Delta$. As a technical note, it is important that $\sgn(0) = 1$ (or $-1$) and \emph{not} equal to $0$. This choice represents the fact that the solution to the inner maximization problem for each individual squared term $\left(\langle x_i + \delta_i, w\rangle - y_i\right)^2$ is nonzero even if $x_i^{\top}w - y_i = 0$.

%, and gives the factor of $n$ in the third term of Form~\ref{equ:objform2}. A priori the factor of $n$ in the third term might suggest that as the number of samples increases the solution is dragged toward  While at first glance the factor of $n$ in the third term might suggest that as the number of samples increases the solution is dragged toward the origin, this is \emph{not} the case, as both the first and second terms implicitly contain a factor of $n$. 

At first glance the objective looks similar to ridge regression or Lasso, particularly when we consider $L_2$- and $L_{\infty}$-adversarial training for which the dual norms are $L_2$ and $L_1$ respectively. However the solutions to this objective are not, in general, identical to the ridge regression or Lasso solutions. In Section~\ref{sec:lossgeometry} we will show how the second term $\epsilon\|w\|_{*}\|Xw - y\|_{1}$ influences the geometry of the loss landscape when $\|\cdot\|_{*} = \|\cdot\|_{2}$. 

\subsection{The Geometry of the Loss Landscape}
\label{sec:lossgeometry}
For the remainder of the paper we will \emph{exclusively} analyze the case where $\Delta$ is an $L_2$-ball, leaving the case of $L_{\infty}$ for future work. We define the loss of interest

\begin{equation}
\label{equ:advtrainloss}
\mathcal{L}_2(X, y; w) = \frac{1}{2}\|Xw - y\|_{2}^{2} + \epsilon\|w\|_{2}\|Xw - y\|_{1} + \frac{\epsilon^2n}{2}\|w\|_{2}^2.
\end{equation}

To build intuition, suppose that $Xw = y$ is an underdetermined system. (Our results will not depend on this assumption.) The set of solutions is given by the affine subspace $S = \{X^{\dagger}y + u: u \in \nullspace(X)\}$, where $X^{\dagger}$ is the pseudo-inverse. The first thing to notice about $\epsilon\|w\|_{2}\|Xw - y\|_{1}$ is that, on its own, it is non-convex, having local minima both at the origin and in $S$. Along any path starting at the origin and ending at a point in $S$, the loss landscape induced by $\epsilon\|w\|_{2}\|Xw - y\|_{1}$ is negatively curved. 

The second thing to notice about $\epsilon\|w\|_{2}\|Xw - y\|_{1}$ is that it is non-smooth. To understand the loss landscape of Equation~\ref{equ:advtrainloss}, it is crucial to understand where $\|Xw - y\|_{1}$ is non-smooth. The term $\|Xw - y\|_{1} = \sum_{i} |x_i^{\top}w - y_i|$ is non-smooth at any point $w$ where some $x_i^{\top}w - y_i = 0$. Geometrically, $x_i^{\top}w - y_i = 0$ is the equation of a hyperplane $h_i$ with normal vector $x_i$ and bias $y_i$. The hyperplane $h_i$ partitions $\R^d$ into two halfspaces $h_{i}^{+}$, $h_i^{-}$ such that every point $w \in h_i^{+}$ has $\sgn(x_i^{\top}w -y_i) = 1$ and $w \in h_i^{-}$ has $\sgn(x_i^{\top}w -y_i) = -1$. The set of hyperplanes $\{h_i: i \in [n]\}$ define a \emph{hyperplane arrangement} $\mathcal{H}$, a subdivision of $\R^d$ into convex cells. See Figure~\ref{fig:objvis}. Let $\mathcal{C} \in \mathcal{H}$ be a cell of the hyperplane arrangement\footnote{We use ``cell'' to refer to a $d$-dimensional face $\mathcal{H}$. When considering a lower dimensional face of $\mathcal{H}$ we will refer to the dimension explicitly.}. Every point $w$ in the interior $\Int{\mathcal{C}}$ of $\mathcal{C}$ lies on the same side of every hyperplane $h_i$ as every other point in $\Int{\mathcal{C}}$. Thus we can identify each $\mathcal{C}$ with a \emph{signature} $s = \sgn(Xw - y)$ for any $w \in \Int{\mathcal{C}}$. 

\begin{figure}[h!]
\begin{center}
\begin{subfigure}{0.245\textwidth}
\includegraphics[width=0.99\linewidth]{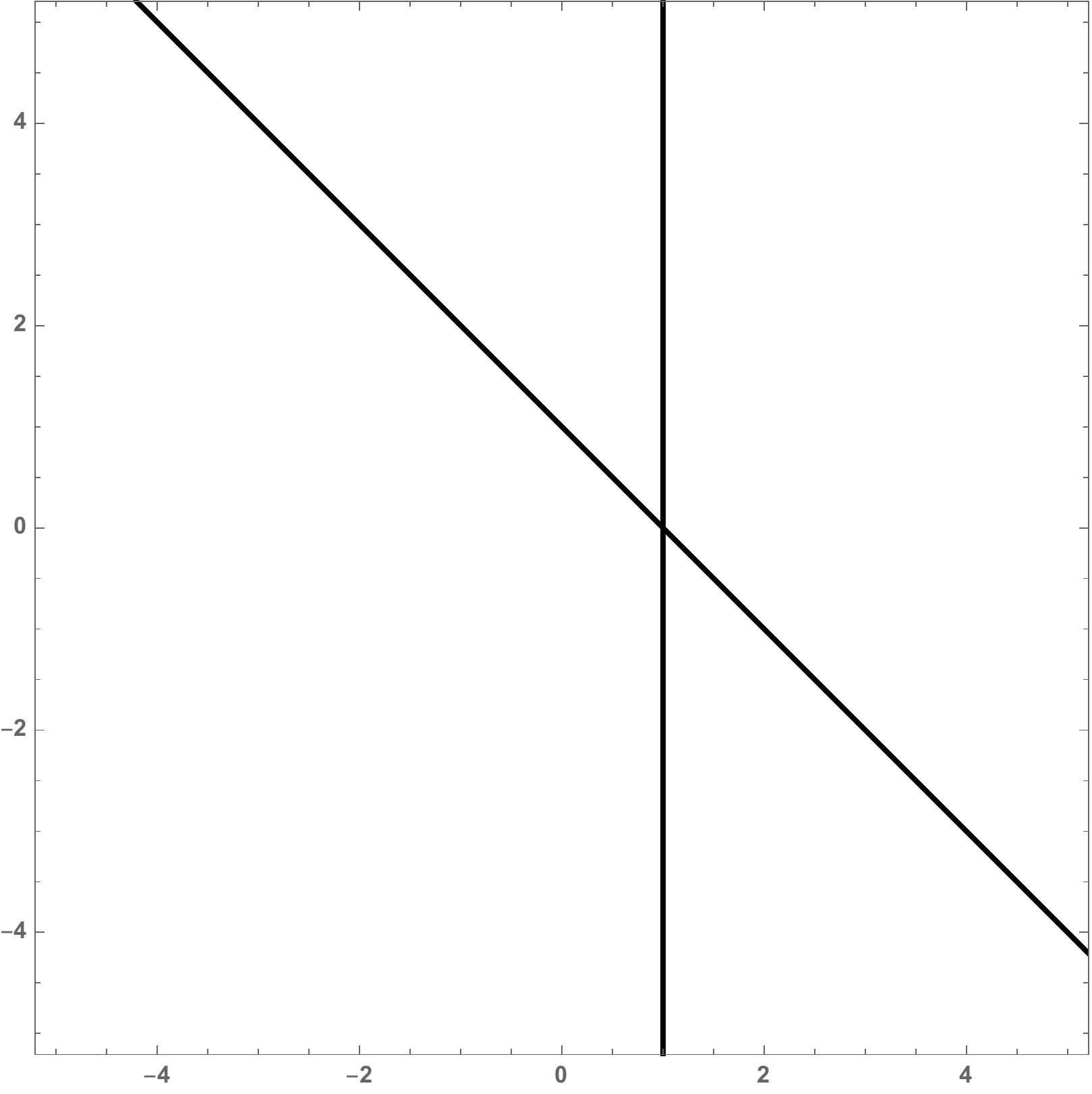}
\end{subfigure}
\begin{subfigure}{0.245\textwidth}
\includegraphics[width=0.99\linewidth]{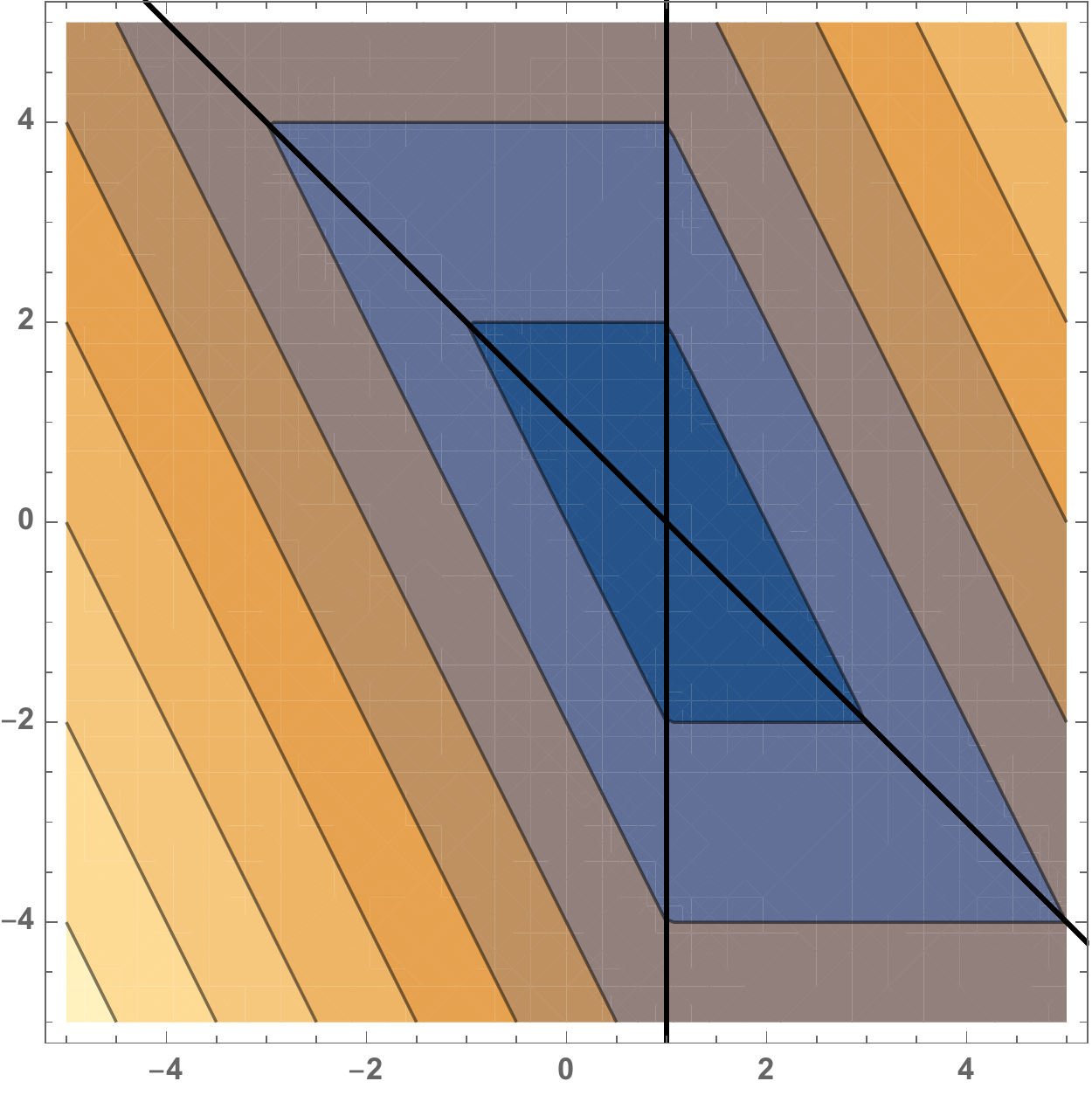}
\end{subfigure}
\begin{subfigure}{0.245\textwidth}
\includegraphics[width=0.99\linewidth]{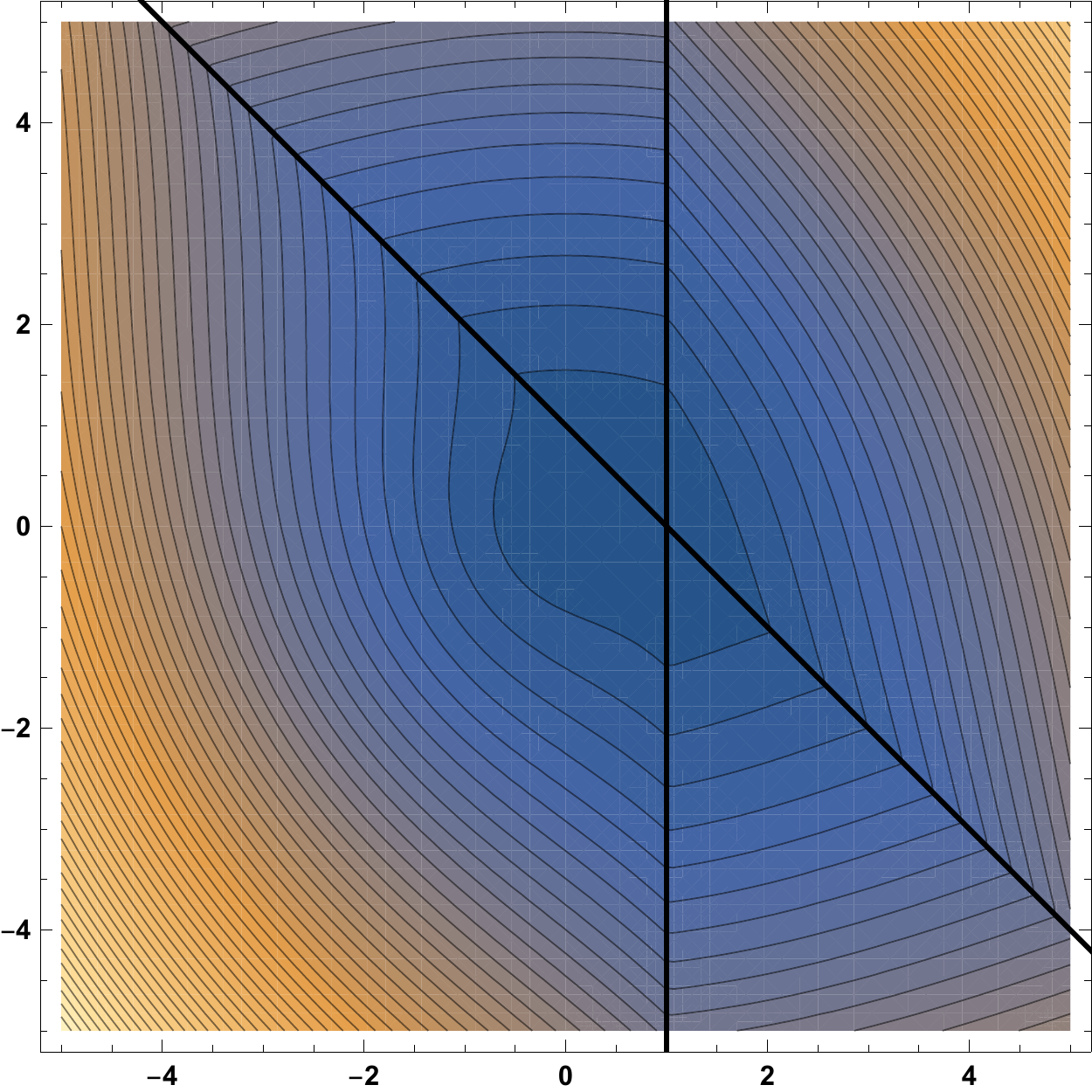}
\end{subfigure}
\begin{subfigure}{0.245\textwidth}
\includegraphics[width=0.99\linewidth]{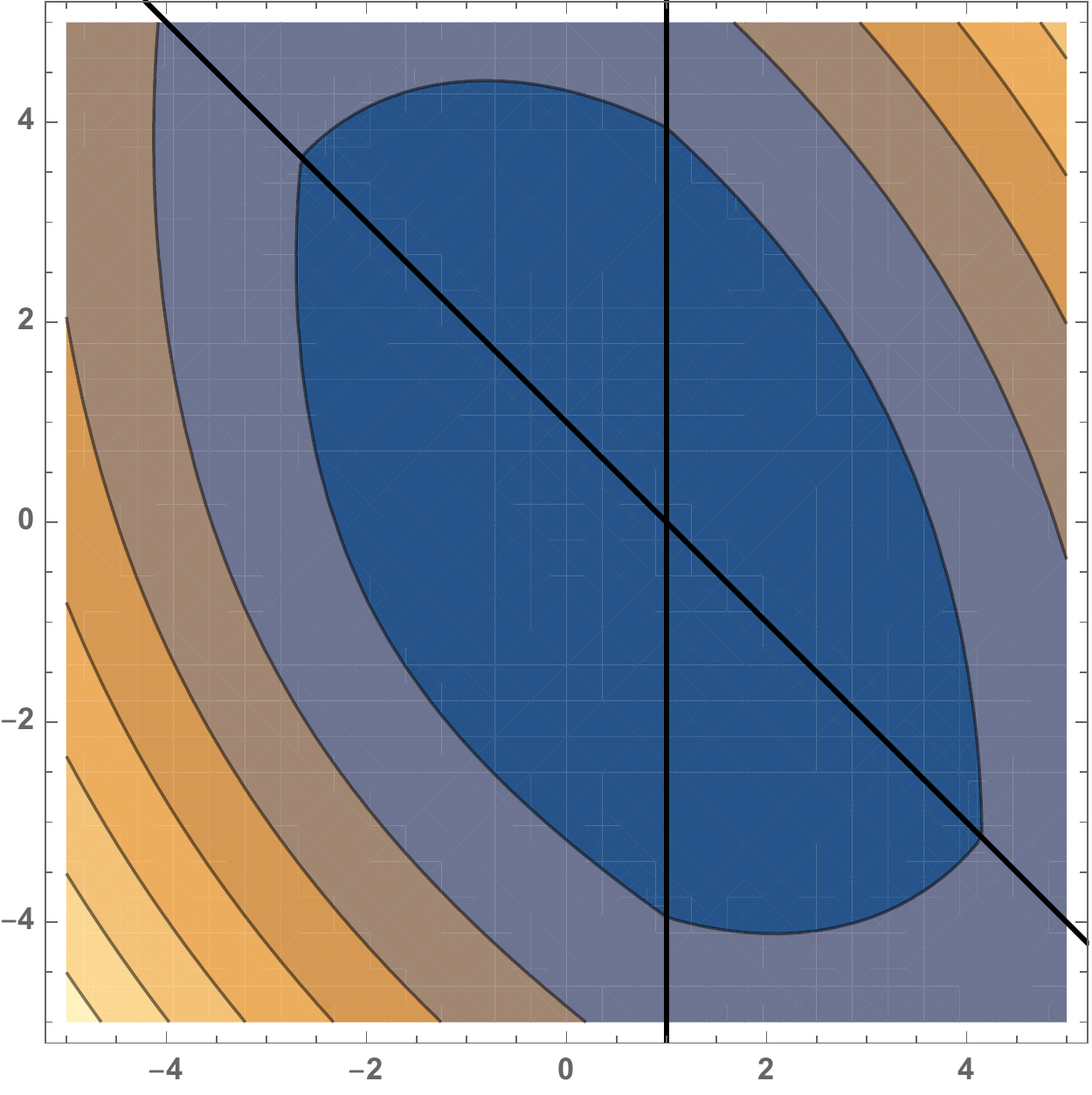}
\end{subfigure}
\begin{subfigure}{0.245\textwidth}
\includegraphics[width=0.99\linewidth]{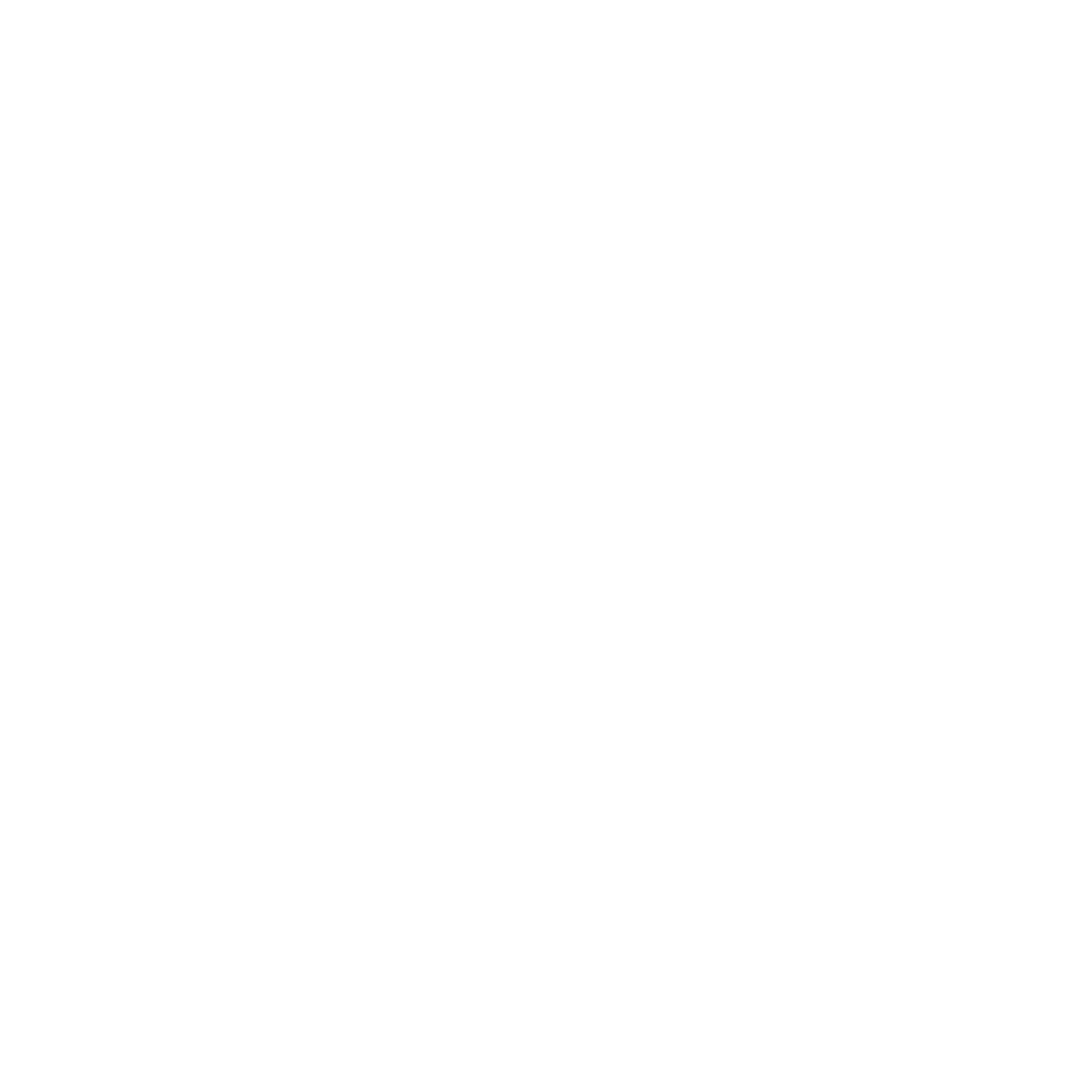}
\end{subfigure}
\begin{subfigure}{0.245\textwidth}
\includegraphics[width=0.99\linewidth]{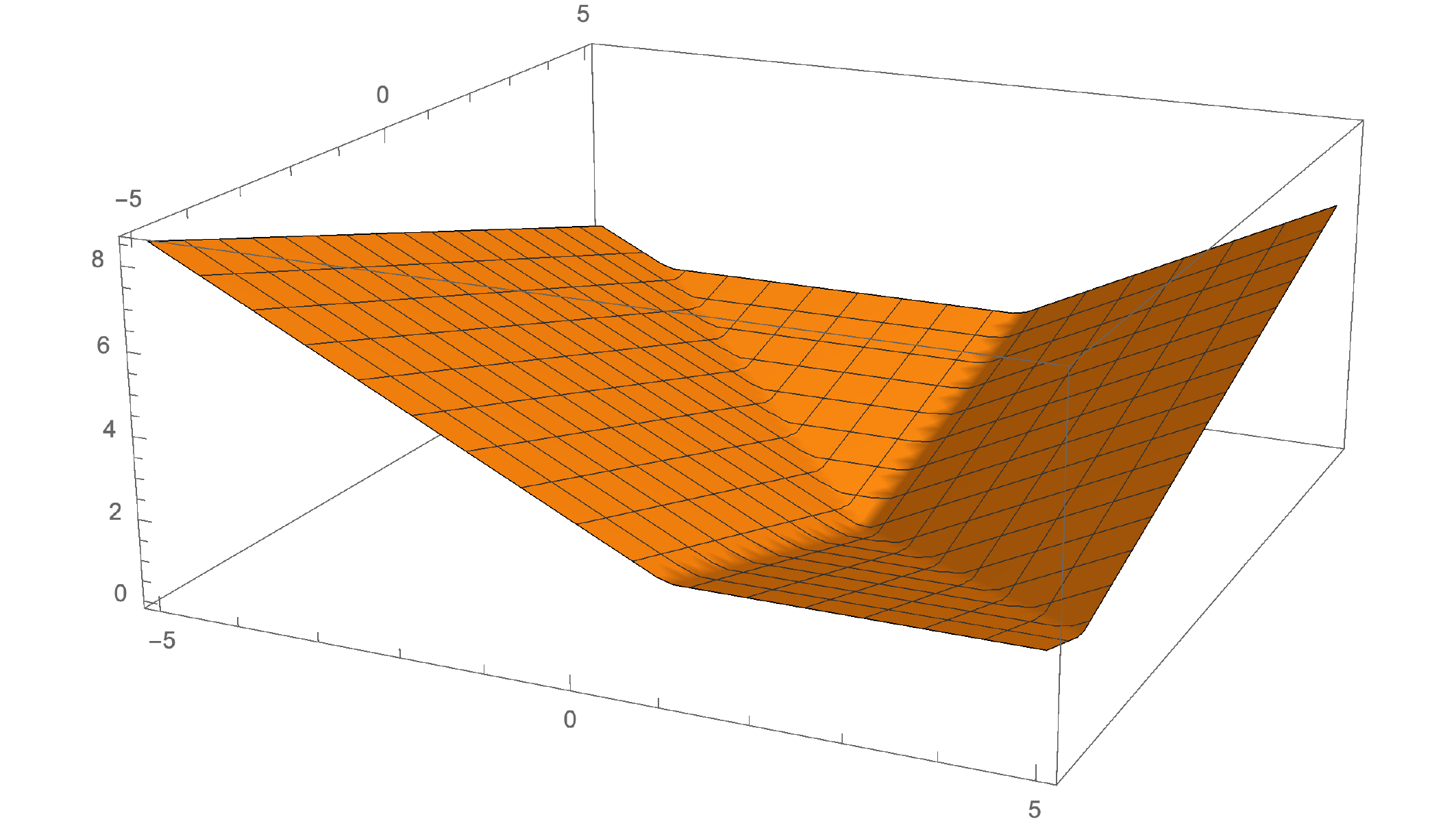}
\end{subfigure}
\begin{subfigure}{0.245\textwidth}
\includegraphics[width=0.99\linewidth]{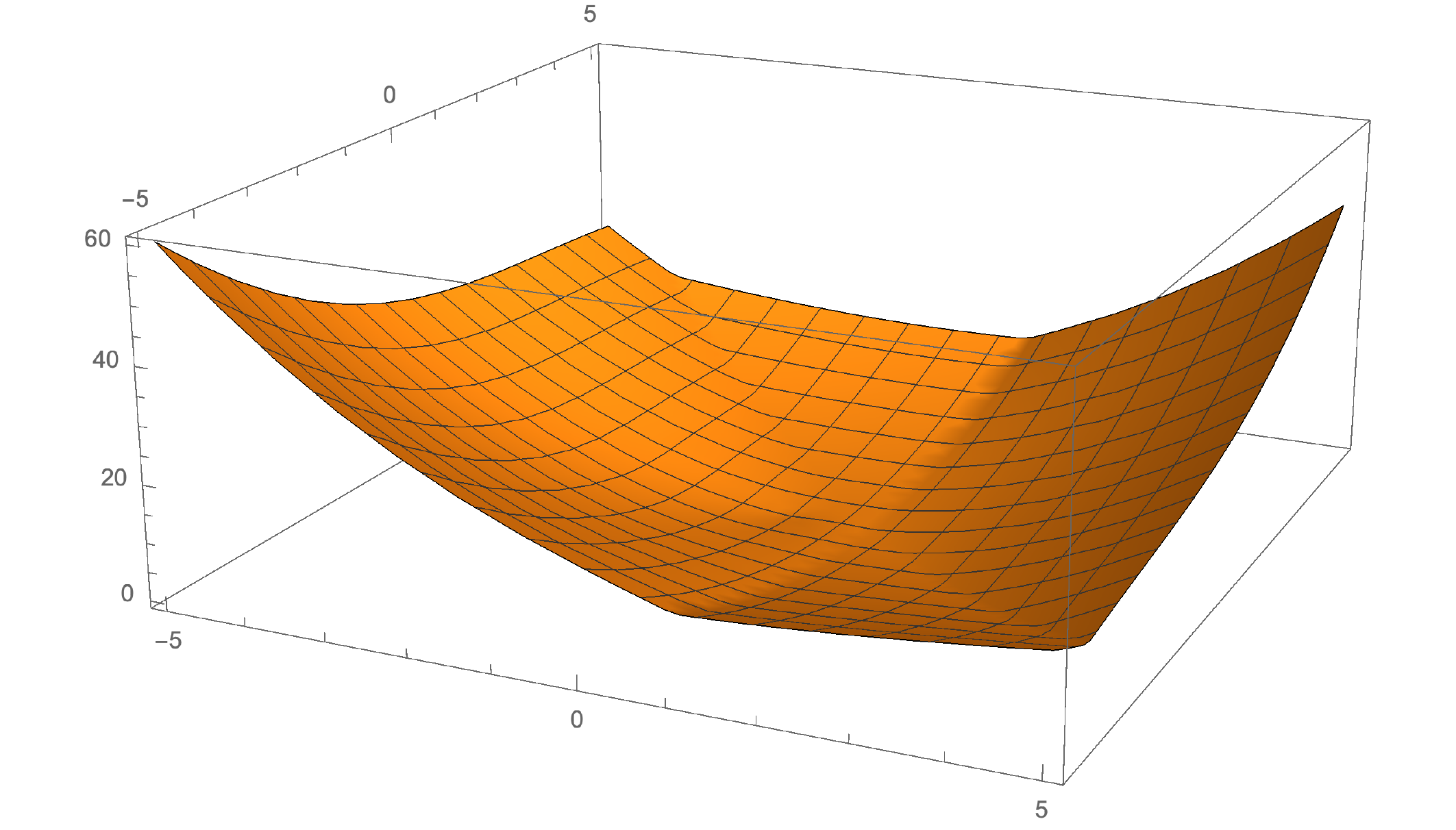}
\end{subfigure}
\begin{subfigure}{0.245\textwidth}
\includegraphics[width=0.99\linewidth]{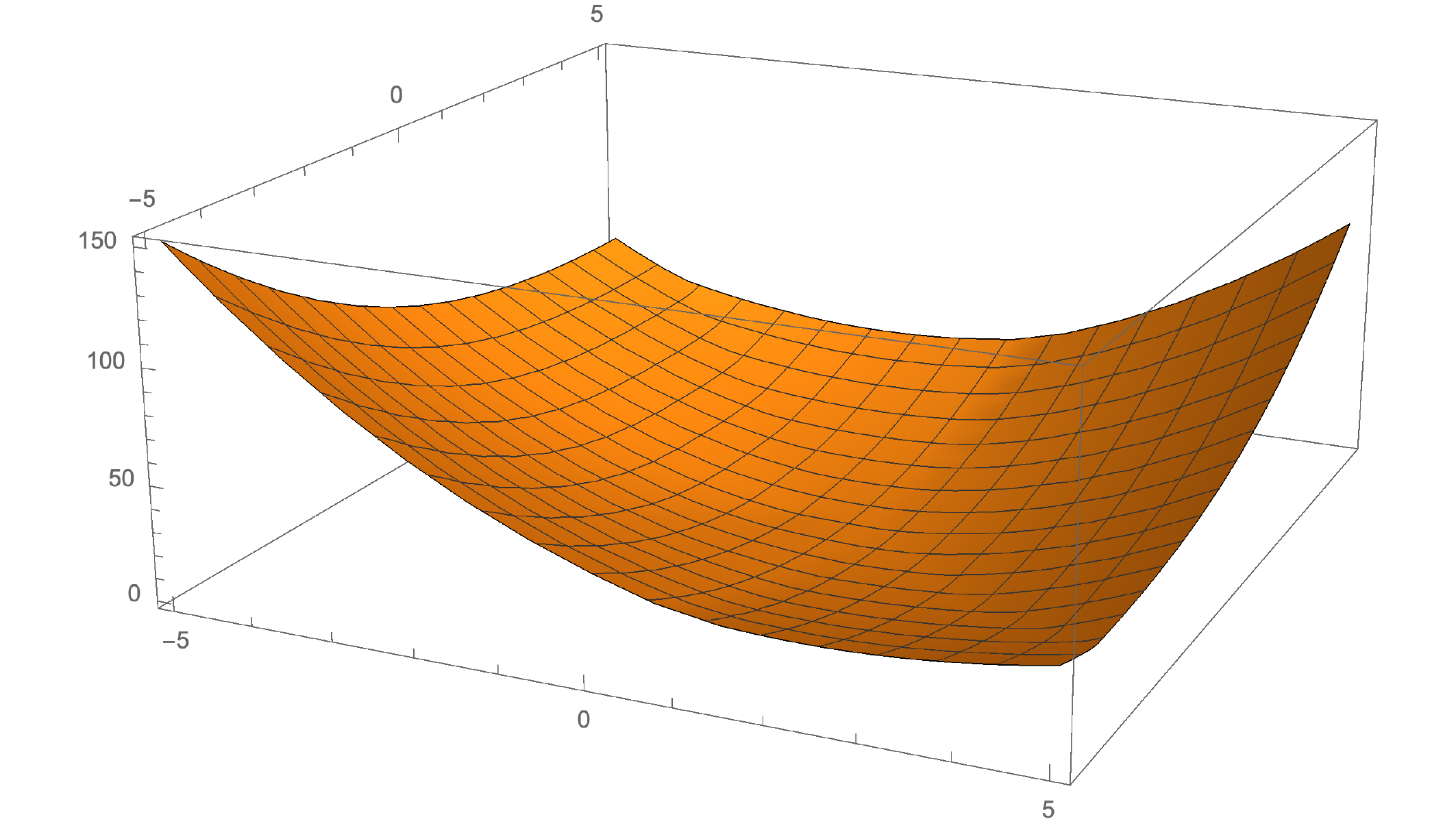}
\end{subfigure}
\caption{The top leftmost figure shows a hyperplane arrangement with two lines that subdivides $\R^2$ into four convex cells. The top center left figure shows the isocontours of $\|Xw - y\|_{1}$. Within each convex cell, the isocontours behave as the linear function $s^{\top}(Xw-y)$, where $s$ is the signature of the cell, and are non-smooth along the two black lines. The top center right figure shows the isocontours of $\epsilon\|w\|_{2}\|Xw - y\|_{1}$, which are clearly non-convex. The top rightmost figure shows the isocontours of the function $\mathcal{L}_{2} = \frac{1}{2}\|Xw - y\|_{2}^{2} + \epsilon\|w\|_{2}\|Xw - y\|_{1} + \frac{\epsilon^2n}{2}\|w\|_{2}^2$. Notice how these isocontours are non-smooth along the two black lines and the asymmetry of the isocontours caused by the $\epsilon\|w\|_{2}\|Xw - y\|_{1}$ term. The bottom row shows the graphs of the functions in the top row.}
\label{fig:objvis}
\end{center}
\end{figure}

Theorem~\ref{thm:advtrainhelps}, our main result of this section, fully characterizes the geometry of Equation~\ref{equ:advtrainloss}. 
\begin{theorem}
\label{thm:advtrainhelps}
$\mathcal{L}_{2}$ is always a convex function, whose optimal solution(s) always lie in $\rowspace(X)$. There are four possible cases, three of which depend on the value of $\epsilon$. 

\begin{enumerate}
    \item If $Xw = y$ is an inconsistent system, then $\mathcal{L}_{2}$ is a strictly convex function.
    \item If $Xw = y$ is a consistent system and $\epsilon \in (0,1/\|X^{\dagger}y\|_{2})$ then $\mathcal{L}_{2}$ is a convex function. Moreover, $\mathcal{L}_{2}$ is strictly convex everywhere except along two line segments, both of which have one endpoint at the origin and terminate at $X^{\dagger}y \pm u$  respectively, for some $u \in \nullspace(X)$. The gradient at every point on these line segments is nonzero, and so the optimal solution is found in $\rowspace(X)$ at a point of strict convexity.
    \item If $Xw = y$ is a consistent system and $\epsilon = 1/\|X^{\dagger}y\|_{2}$, then $\mathcal{L}_{2}$ is a convex function. Moreover, $\mathcal{L}_{2}$ is strictly convex everywhere except along a single line segment with one endpoint at the origin and the other endpoint at $X^{\dagger}y$. The optimal solution(s) may or may not lie along this line. 
    \item If $Xw = y$ is a consistent system and $\epsilon > 1/\|X^{\dagger}y\|_{2}$, then $\mathcal{L}_2$ is a strictly convex function.
\end{enumerate}

Furthermore, $\mathcal{L}_2$ is subdifferentiable everywhere. 
\end{theorem}

The proof of Theorem~\ref{thm:advtrainhelps} first characterizes the geometry of $\mathcal{L}_{2}$ restricted to the interior of each convex cell $\mathcal{C} \in \mathcal{H}$. We denote this function by $\mathcal{L}_{2}|_{\Int{\mathcal{C}}}$. If the signature $s$ of $\mathcal{C}$ is not equal to $-y$, meaning that $\mathcal{C}$ does not contain the origin, then $\mathcal{L}_{2}|_{\Int{\mathcal{C}}}$ is always strongly convex. The cell $\mathcal{C}$ with signature $s = -y$ is the only cell in which $\mathcal{L}_{2}$ might not be strongly convex. The cases in Theorem~\ref{thm:advtrainhelps} correspond to the cases that characterize the geometry of $\mathcal{L}_{2}|_{\Int{\mathcal{C}}}$ for $s = -y$. The transitions between cells are strictly convex and the subdifferential is non-empty at these transitions.  

We find it very interesting that $\mathcal{L}_{2}$ is strictly or strongly convex almost everywhere. For $\epsilon \in (0, 1/\|X^{\dagger}y\|_{2})$, $\mathcal{L}_{2}$ is convex, but not strictly convex, only along two line segments which lie outside of the rowspace of $X$. The gradient along these line segments is nonzero, and so this particular type of convexity does not prevent an optimization algorithm from finding the \emph{unique} solution in the rowspace of $X$.  For $\epsilon = 1/\|X^{\dagger}y\|_{2}$, $\mathcal{L}_{2}$ is convex, but not strictly convex, only along a single line segment in the rowspace of $X$. However the condition $\epsilon \neq 1/\|X^{\dagger}y\|_{2}$ can be ensured by an infinitesimal perturbation. The following remark is immediate.

\begin{cor}
\label{cor:advtrainhelps}
Suppose $\epsilon \neq 1/\|X^{\dagger}y\|_{2}$. Then any optimization algorithm which is guaranteed to converge to the global minimum for a strictly convex subdifferentiable function and which does not prematurely terminate at a point with nonzero gradient finds the unique global minimum of $\mathcal{L}_{2}$.
\end{cor}

Corollary~\ref{cor:advtrainhelps} states that, in the linear case, any reasonable optimization algorithm finds the unique global optimum of $\mathcal{L}_{2}$, almost always.\footnote{The condition on ``premature termination'' in Corollary~\ref{cor:advtrainhelps} is meant to rule out the following case. One could construct an optimization algorithm that is guaranteed to converge for strictly convex functions, but terminates early upon detecting a point at which there exists a direction in which the function is convex but not strictly convex, even if the gradient is nonzero and the global minimum is at a point of strict convexity. We doubt any commonly used optimization algorithm would have difficulty with the geometry of $\mathcal{L}_{2}$.}  We conclude that, in the linear case, $L_{2}$-adversarial training does indeed sufficiently regularize the loss landscape so that any optimization algorithm finds the same solution. 

%A priori, the non-smoothness introduced by adversarial training may be unexpected by practitioners. While subgradient descent methods can in theory find the optimum of $\mathcal{L}_{2}$, it is more likely that algorithms which utilize auto-differentiation, as currently implemented, oscillate between cells near the optima, in cases where the optima lies at a point of non-differentiability. 

Unfortunately Theorem~\ref{thm:advtrainhelps} does not give a closed-form expression for the solution(s) of $\mathcal{L}_{2}$. In Lemma~\ref{lem:subdiff} we characterize the subdifferential $\partial \mathcal{L}_{2}(w)$ at every point. However solving for $w$ where $0 \in \partial \mathcal{L}_{2}(w)$ is similar to solving a linear program, and so we suspect that no closed-form solution exists. In Section~\ref{sec:charsol} we discuss the solution for $\mathcal{L}_{2}$ in the particular case of the learning problem defined in Section~\ref{sec:adpvsgd} and show that the max-margin solution is often \emph{not} the solution recovered by adversarial training. 

\subsection{The Solutions to the Learning Problem of Section~\ref{sec:adpvsgd}}
\label{sec:charsol}

While we know of no technique to characterize the set of solutions for $\mathcal{L}_{2}$ in general, we can still make some statements about the solution in specific instances, such as the learning problem described in Section~\ref{sec:adpvsgd}. First, since the solution(s) of $\mathcal{L}_{2}$ must lie in the rowspace, the ``obvious'' solution $w^{*}$ to the learning problem in Section~\ref{sec:adpvsgd} is not recovered by $L_2$-adversarial training. A priori, one might guess that the minimum $L_{2}$-norm solution $X^{\top}\alpha$ is the solution to $\mathcal{L}_{2}$ . However this is only true under specific conditions which depend on the class imbalance.  

\begin{theorem}
\label{thm:minnormsoladv}
Let $(X, y)$ be the learning problem defined in Section~\ref{sec:adpvsgd}. $X^{\top}\alpha$ is a solution to $\mathcal{L}_{2}$ if and only if 
\begin{equation}
\label{equ:epscond1}
     \epsilon \leq \frac{\sqrt{64n_{+}^2n_{-}^2 + 160n_{+}^2n_{-} + 75n_{+}^2 + 32n_{+}n_{-}^2+60n_{+}n_{-}+70n_{+} + 3n_{-}^2 + 5n_{-}}}{\max\left(4n_{-}^2 + 4n_{-}n_{+}+5n_{+}+5n_{-},4n_{+}^2 + 4n_{-}n_{+}+n_{+}+n_{-}\right)}.
\end{equation}
Let $c > 0$ be a constant such that $n_{+} = cn_{-}$. If $\epsilon \leq \min\left\{\frac{2c}{1+c}, \frac{2}{1+c}\right\}$ then $X^{\top}\alpha$ is a solution to $\mathcal{L}_{2}$.
\end{theorem}

While we need the first condition, which is both necessary and sufficient, to draw our forthcoming conclusion, the second, merely sufficient, condition provides greater intuition. The first condition states that $X^{\top}\alpha$ is a solution if and only if $\epsilon$ is sufficiently small, as a function of $n_{+},n_{-}$. For the learning problem in Section~\ref{sec:adpvsgd}, we know that $\epsilon = 1$ is achievable, so it is natural to ask how large an $\epsilon$ is allowable by Equation~\ref{equ:epscond1}. This relationship depends on the class imbalance, and so we set $n_{+} = cn_{-}$ and derive the condition in the second part of the proof of Theorem~\ref{thm:minnormsoladv}, which is a lower bound on the right-hand side of Equation~\ref{equ:epscond1}. The term $\min\left\{\frac{2c}{1+c}, \frac{2}{1+c}\right\}$ is at most $1$, when $c=1$, but can be arbitrarily less than $1$ depending on $c$; see Figure~\ref{fig:bound}. We note also that the gap between the lower bound $\min\left\{\frac{2c}{1+c}, \frac{2}{1+c}\right\}$ and the right-hand side of Equation~\ref{equ:epscond1} is already small for $n_{-} \approx 20$ and vanishes as $n_{-} \rightarrow \infty$. Thus we conclude that if $\epsilon$ is close to $1$ and the dataset is even moderately imbalanced, $X^{\top}\alpha$, which maximizes the $L_{2}$-margin, is \emph{not} a solution for $\mathcal{L}_{2}$.

\begin{figure}[h!]
\begin{center}
\includegraphics[width=0.49\linewidth]{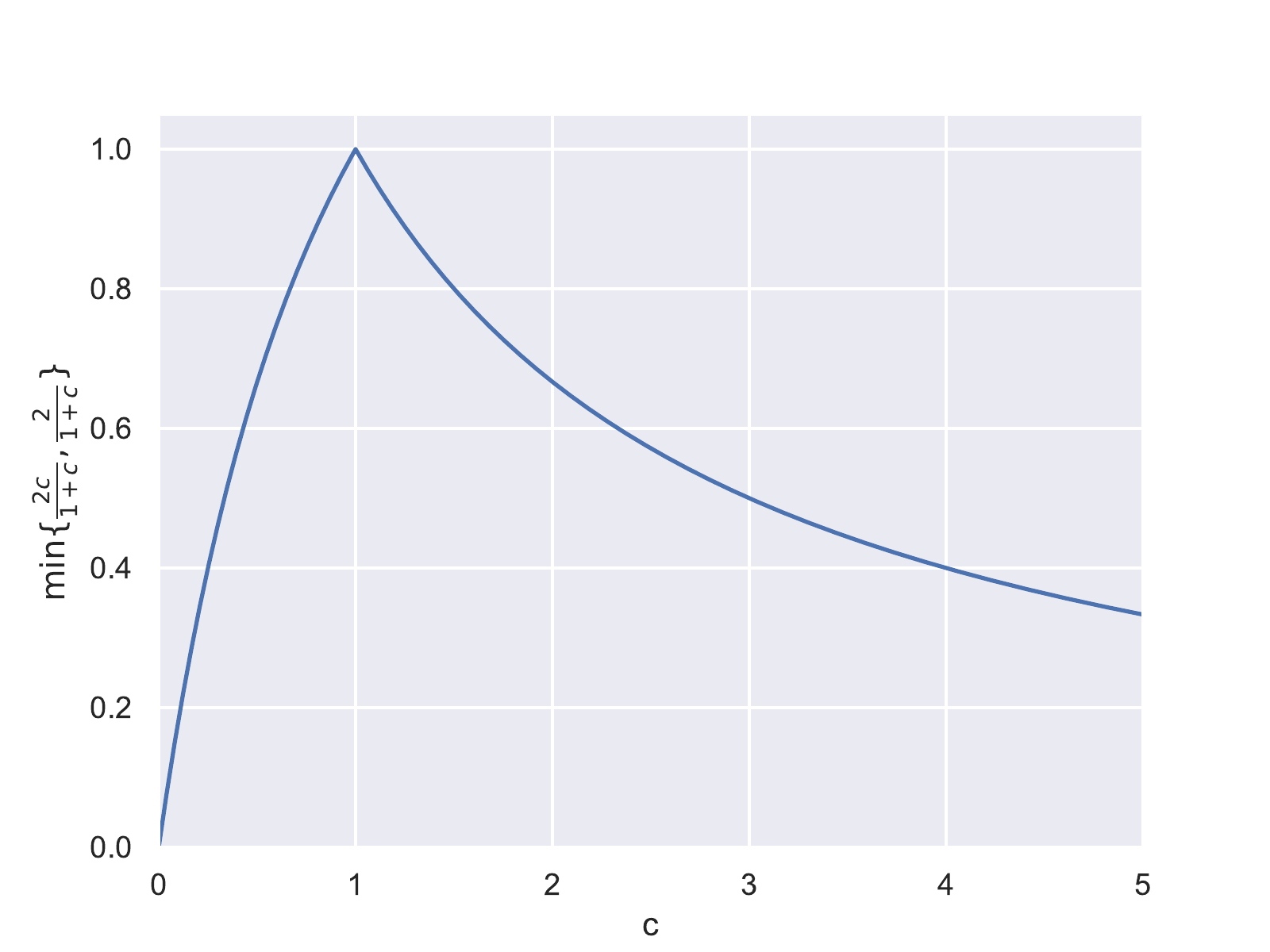}
\caption{The bound  $\min\left\{\frac{2c}{1+c}, \frac{2}{1+c}\right\}$ as a function of the class imbalance $c$. If the classes are perfectly balanced ($c = 1$), then we can take $\epsilon$ up to $1$ and recover the minimum $L_{2}$-norm solution $X^{\top}\alpha$. As the class imbalance increases the maximum allowable $\epsilon$ for which we recover $X^{\top}\alpha$ decreases rapidly.}
\label{fig:bound}
\end{center}
\end{figure}

\section{Experiments}
In this section we experimentally explore the effect of different optimization algorithms on robustness for deep networks. We are interested in the following questions. (1) Do different optimization algorithms give qualitatively different robustness results? (2) Does adversarial training reduce or eliminate the influence of the optimization algorithm? (3) Do adaptive or non-adaptive methods consistently outperform the other for both $L_{2}$- and $L_{\infty}$-adversarial attacks, even if the difference is small?

\paragraph{Models} 
For MNIST our model consist of two convolutional layers with 32 and 64 filters respectively, each followed by $2 \times 2$ max pooling. After the two convolutional layers, there are two fully connected layers each with 1024 hidden units. For CIFAR10 we use a ResNet18 model (\cite{He16}). We use the same model architectures for both natural and adversarial training. These models were chosen because they are small enough for us to run a large hyperparameter search.

\paragraph{Parameters for Adversarial Training}
For adversarial training we use the approach of \cite{Madry17}, and train against a projected gradient descent (PGD) adversary. For MNIST with  $L_{\infty}$-adversarial training, we train against a $40$-step PGD adversary with step size $0.01$ and maximum perturbation size of $\epsilon = 0.3$. For $L_{2}$-adversarial training we train against a $40$-step PGD adversary with step size $0.05$ and maximum perturbation size of $\epsilon = 1.5$.
For CIFAR10 with $L_{\infty}$-adversarial training, we train against against a $10$-step PGD adversary with step size $0.007 (= 2/255)$ and maximum perturbation size of $\epsilon= 0.031 (= 8 / 255)$. For $L_{2}$-adversarial training we train against a $10$-step PGD adversary with step size $0.039 (= 10/255)$ and a maximum perturbation size of  $\epsilon= 0.117 (= 30 / 255)$.

\paragraph{Attacks for Evaluation}
On MNIST we apply $100$-step PGD with $10$ random restarts. For $L_{\infty}$ we apply PGD with step sizes $\{0.01, 0.05, 0.1, 0.2\}$, and for $L_{2}$ we apply PGD with step sizes $\{0.05, 0.1, 10^5, 10^9\}$. On CIFAR10 we apply $20$-step PGD with $5$ random restarts. For $L_{\infty}$ we apply PGD with step sizes $\{2 / 255, 3 / 255, 4/255\}$ and for $L_{2}$ we apply PGD with step sizes $\{10 / 255, 20/255, 10^5\}$. We also apply the gradient-free BoundaryAttack++ (\cite{Chen19}). We evaluate these attacks \emph{per sample}, meaning that if any attack successfully constructs an adversarial example for a sample $x$ at a specific $\epsilon$, it reduces the robust accuracy of the model at that $\epsilon$.

\paragraph{Metrics}
We plot the robust classification accuracy for each attack as a function of $\epsilon \in [0, \epsilon_{\max}]$. We are interested in both natural and adversarial training. Usually when heuristic methods for adversarial training are evaluated, they are compared at the specific $\epsilon$ for which the model was adversarially trained. Such a comparison is arbitrary for naturally trained models and is also unsatisfying for adversarially trained models. 
To compare the robustness of different optimization algorithms we instead consider the area under the robustness curve. Following \cite{Khoury19}, we report the \emph{normalized area under the curve} (NAUC) defined as 
\begin{equation}
\operatorname{NAUC}(\operatorname{acc}) = \frac{1}{\epsilon_{\max}}\int_{0}^{\epsilon_{\max}} \operatorname{acc}(\epsilon) \; d\epsilon,
\end{equation}
where $\operatorname{acc}: [0, \epsilon_{\max}] \rightarrow [0, 1]$ measures the classification accuracy. Note that NAUC $\in [0,1]$ with higher values corresponding to more robust models.

\paragraph{Hyperparameter Selection}
We perform an extensive hyperparameter search over the learning rate (and if applicable momentum) parameter(s) of a each optimization algorithm to identify parameter settings which produce the most robust models. For each dataset we set aside a validation set of size 5000 from the training set. We then train models, with the architecture described above, for each of the hyperparameter settings described below for 100 epochs. All optimization algorithms are started from the same initialization. We evaluate the robustness of each model as described above on the validation set. The hyperparameter settings for each optimization algorithm with achieve the largest NAUC are then used to train new models and then evaluated on the full test set. These final results are the ones we report in this section. The hyperparameters we explored were influenced hyperparamter search of \cite{Wilson17}.

\sloppy The following search space is defined for MNIST. For SGD we consider the learning rates $\{2, 1, 0.5, 0.1, 0.01, 0.001, 0.003, 0.0001\}$. For SGD with momentum we consider the set of learning rates for SGD for each of the momentum settings $\{0.9, 0.8, 0.7\}$. For Adam, Adagrad, and RMSprop we consider initial learning rates $\{0.1, 0.01, 0.001, 0.003, 0.0001\}$.

The following search space is defined for CIFAR10. For SGD we consider the learning rates $\{2, 1, 0.5, 0.25, 0.1\}$. For SGD with momentum we consider the set of learning rates for SGD for each of the momentum settings $\{0.9, 0.8, 0.7\}$. For both SGD and SGD with momentum we reduce the learning rate using the \textsc{ReduceLROnPlateau} scheduler in PyTorch. For Adam and RMSprop we consider the initial learning rates $\{0.005, 0.001, 0.0005, 0.0003, 0.0001, 0.00005\}$. For Adagrad we consider initial learning rates $\{0.1, 0.05, 0.01, 0.0075, 0.0005\}$. 

\fussy Unfortunately adversarially training takes an order of magnitude longer than natural training, since in the innermost loop we must perform an iterative PGD attack to construct adversarial examples. Due to our limited resources, we consider only a subset of the hyperparameters above for adversarial training.

\subsection{MNIST}
Figure~\ref{fig:mnistexp} (Top) shows the robustness of naturally trained models to $L_{\infty}$- and $L_{2}$-adversarial attacks on MNIST. Against $L_{\infty}$-adversarial attacks, RMSprop produce the most robust model with NAUC $0.33$, followed by Adam ($0.27$), SGD with momentum ($0.26$), and SGD and Adagrad ($0.22$). Against $L_{2}$-adversarial attacks, SGD produces the most robust model with NAUC $0.49$, followed by SGD with momentum ($0.48$), Adam ($0.43$), and Adagrad and RMSprop ($0.41$).  

Against $L_{\infty}$-adversarial attacks, RMSprop produces a model that appears qualitatively more robust than the next best performing model. This difference can be large at specific $\epsilon$; for example at $\epsilon = 0.2$, the RMSprop model maintains a robust accuracy of $22\%$, while the Adam model has robust accuracy $7\%$. This is the only instance across all of our experiment in which we observe a notable qualitative difference between different algorithms. 

Figure~\ref{fig:mnistexp} (Bottom) shows the robustness of adversarially trained models. Training adversarially improves the robustness over naturally trained models regardless of the optimization algorithm and all optimization algorithms give qualitatively similar results. For $L_{\infty}$-adversarial training, SGD produces the most robust model with NAUC 0.66, followed by SGD with momentum (0.65), Adam (0.64), RMSprop (0.63), and Adagrad (0.61). For $L_{2}$-adversarial training SGD with momentum and RMSprop produce models with NAUC 0.56, followed by SGD (0.54), Adam (0.53), and Adagrad (0.51). For adversarial training on MNIST, either SGD or SGD with momentum were among the top performers, with Adagrad always producing the worst performing model. 

Adversarial training does seem to reduce the dependence on the choice of optimization algorithm, though does not completely remove it. Against $L_{\infty}$-adversarial attacks at $\epsilon =0.3$, the SGD model maintains robust accuracy of $91.5\%$, while the Adagrad model maintains robust accuracy $83.5\%$. We consider this difference noteworthy for MNIST. The difference is less pronounced for $L_{2}$-adversarial training.

SGD and SGD with momentum consistently outperform other optimization algorithms, yielding the best models in three out of four experiments.  (Or one of the most robust models in the case of ties.) While the difference is qualitatively small, we believe that the consistency with which SGD or SGD with momentum produces the most robust model is noteworthy. Furthermore Adagrad seems to consistently under-perform other optimization algorithms. (Though this may depend on the domain~(\cite{Tifrea19})).

\begin{figure}[h!]
\begin{center}
\begin{subfigure}{0.49\textwidth}
\includegraphics[width=0.99\linewidth]{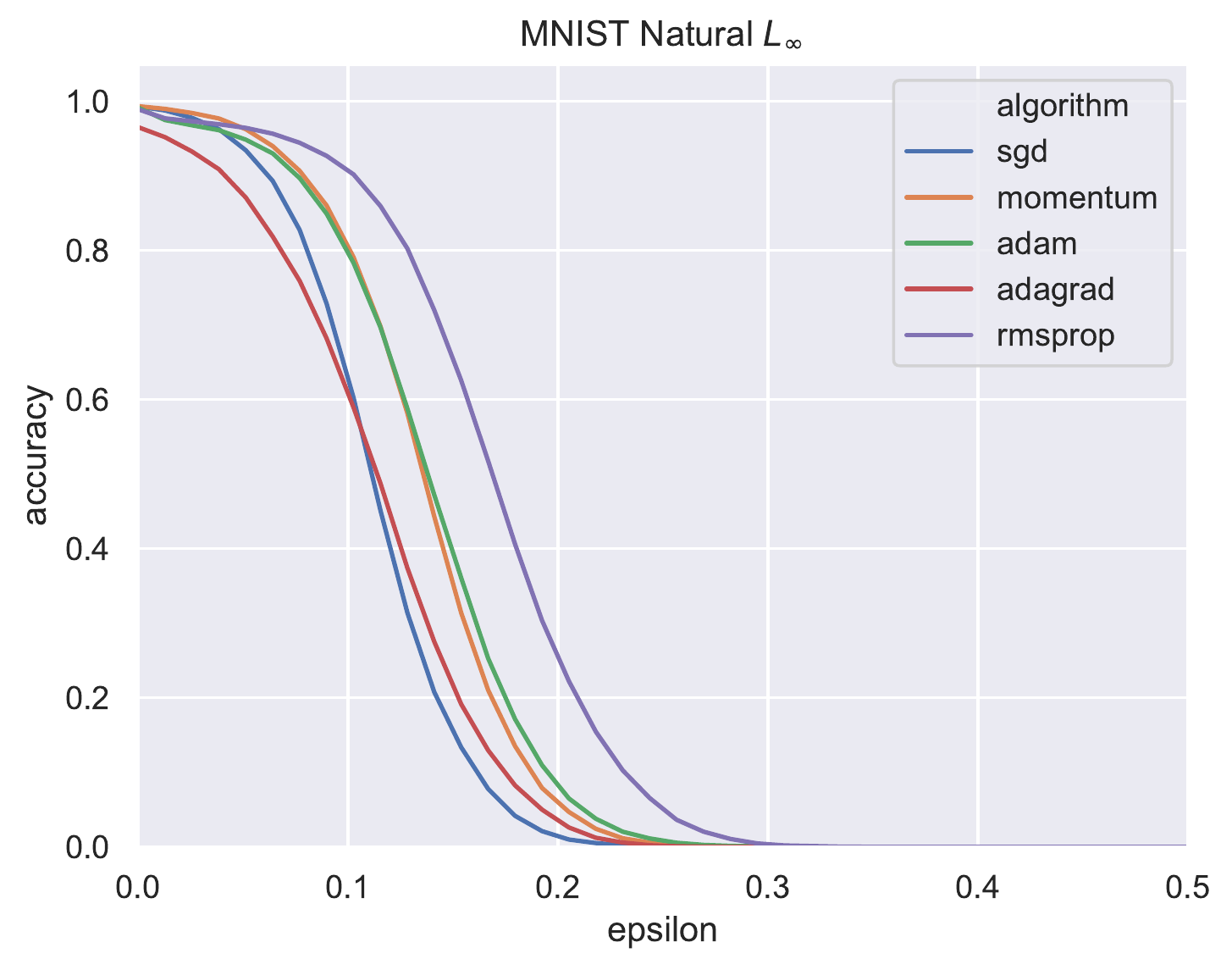}
\end{subfigure}
\begin{subfigure}{0.49\textwidth}
\includegraphics[width=0.99\linewidth]{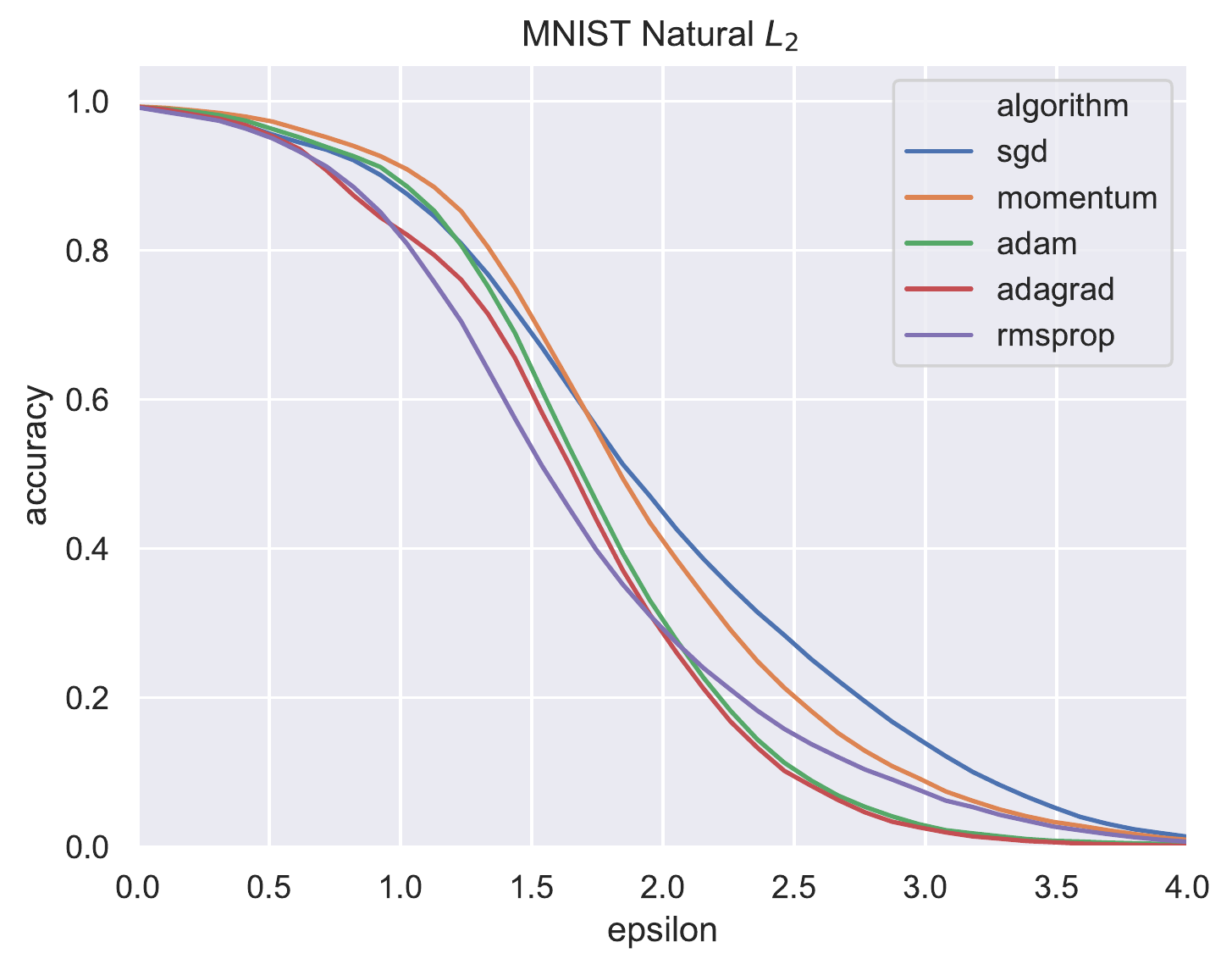}
\end{subfigure}
\begin{subfigure}{0.49\textwidth}
\includegraphics[width=0.99\linewidth]{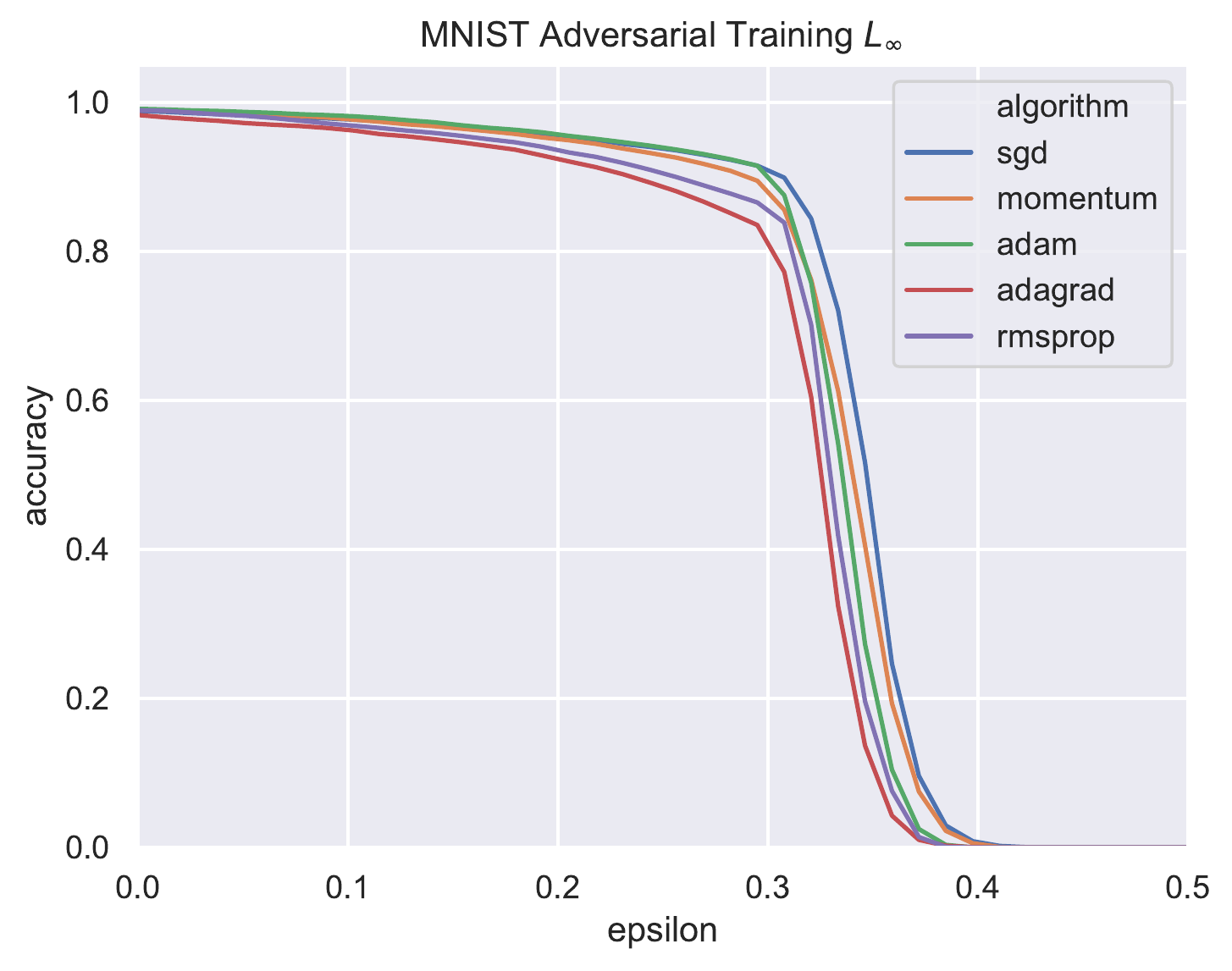}
\end{subfigure}
\begin{subfigure}{0.49\textwidth}
\includegraphics[width=0.99\linewidth]{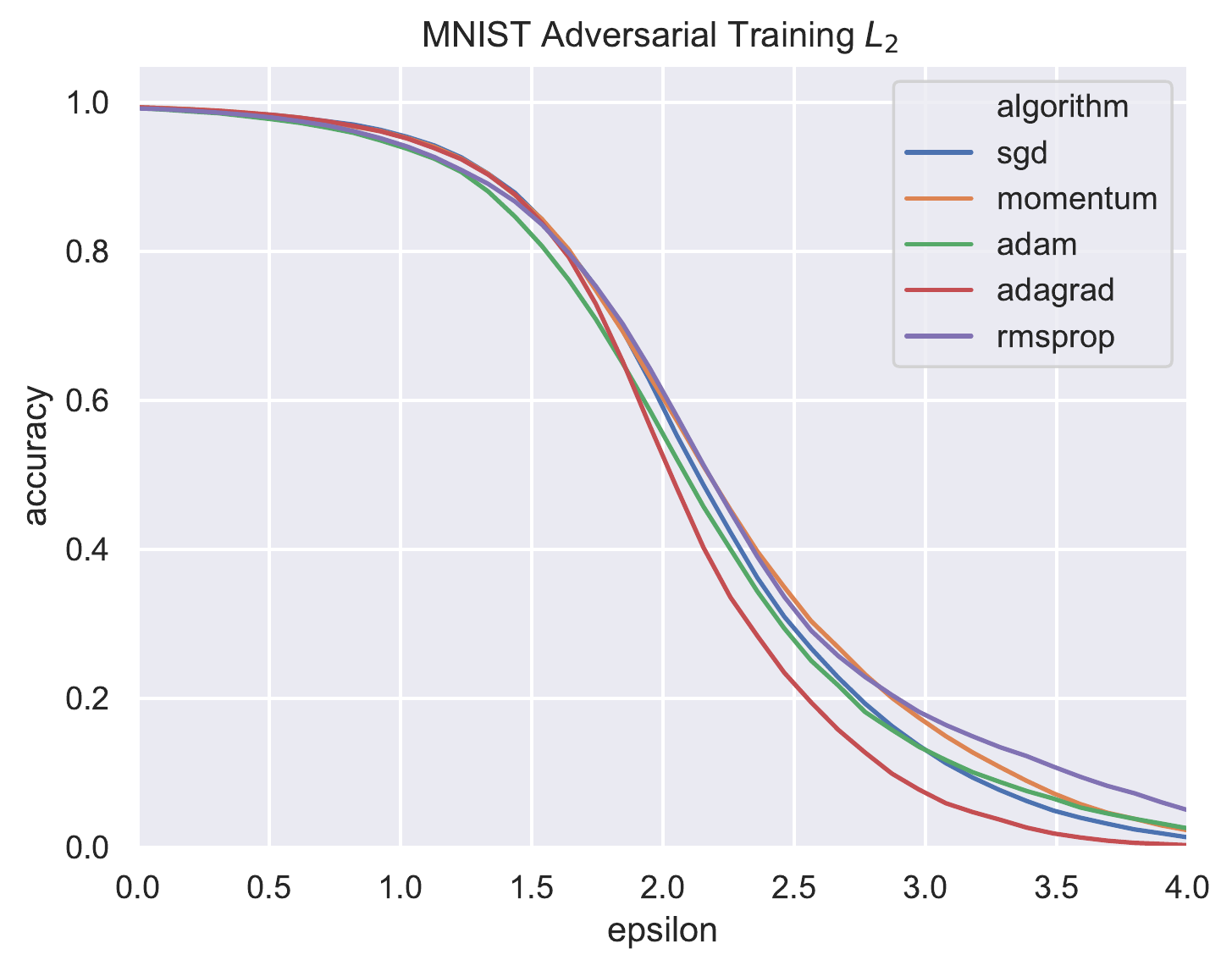}
\end{subfigure}
\caption{\textbf{Top}: Robust accuracy for naturally trained MNIST models against $L_{\infty}$- and $L_{2}$-adversarial attacks. \textbf{Bottom}: Robust accuracy for adversarially trained MNIST models. Left, $L_{\infty}$-adversarially trained models evaluated against $L_{\infty}$-adversarial attacks; right $L_{2}$.  }
\label{fig:mnistexp}
\end{center}
\end{figure}

\subsection{CIFAR10}
Figure~\ref{fig:cifarexp} (Top) shows the robustness of naturally trained models to $L_{\infty}$- and $L_{2}$-adversarial attacks on CIFAR10. Against $L_{\infty}$-adversarial attacks, SGD with momentum produces the most robust model with NAUC $0.06$, followed by RMSprop ($0.05$), Adam and Adagrad ($0.04$), and SGD ($0.02$). Against $L_{2}$-adversarial attacks Adagrad produces the most robust models with NAUC $0.19$, followed by SGD, Adam, and RMSprop ($0.17$), and SGD with momentum ($0.15$).

Figure~\ref{fig:cifarexp} (Bottom) shows the robustness of adversarially trained models.  For $L_{\infty}$-adversarial training, SGD with momentum produced the most robust model with NAUC $0.26$, followed by SGD ($0.24$), Adam ($0.23$), and Adagrad and RMSprop ($0.22$). For $L_{2}$-adversarial training, SGD produces the most robust model with NAUC $0.51$, followed by SGD with momentum ($0.50$), Adam and RMSprop ($0.49$), and Adagrad ($0.47$).

Adversarial training lessens the dependence on the choice of optimization algorithm. Against $L_{\infty}$-adversarial attacks at $\epsilon= 0.031 (=8/255)$, the SGD with momentum model maintains robust accuracy of $45\%$, while the RMSprop model maintains robust accuracy $34\%$. The difference is less pronounced for $L_{2}$-adversarial training. 

SGD or SGD with momentum consistently outperform other optimization algorithms, yielding the best models in three out of four experiments. Again while the difference is qualitatively small, we believe that the consistency with which SGD or SGD with momentum produces the most robust model is noteworthy.

\begin{figure}[h!]
\begin{center}
\begin{subfigure}{0.49\textwidth}
\includegraphics[width=0.99\linewidth]{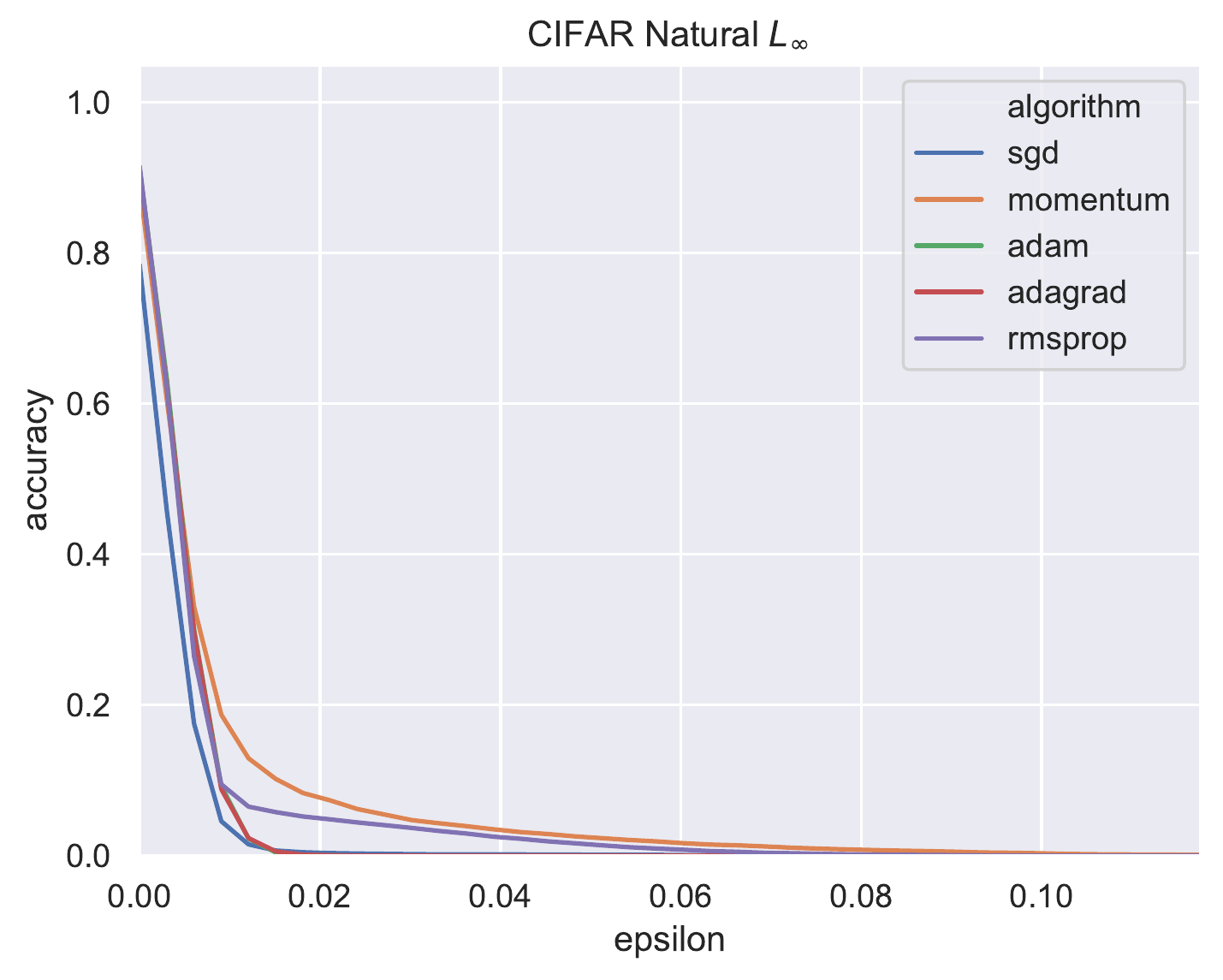}
\end{subfigure}
\begin{subfigure}{0.49\textwidth}
\includegraphics[width=0.99\linewidth]{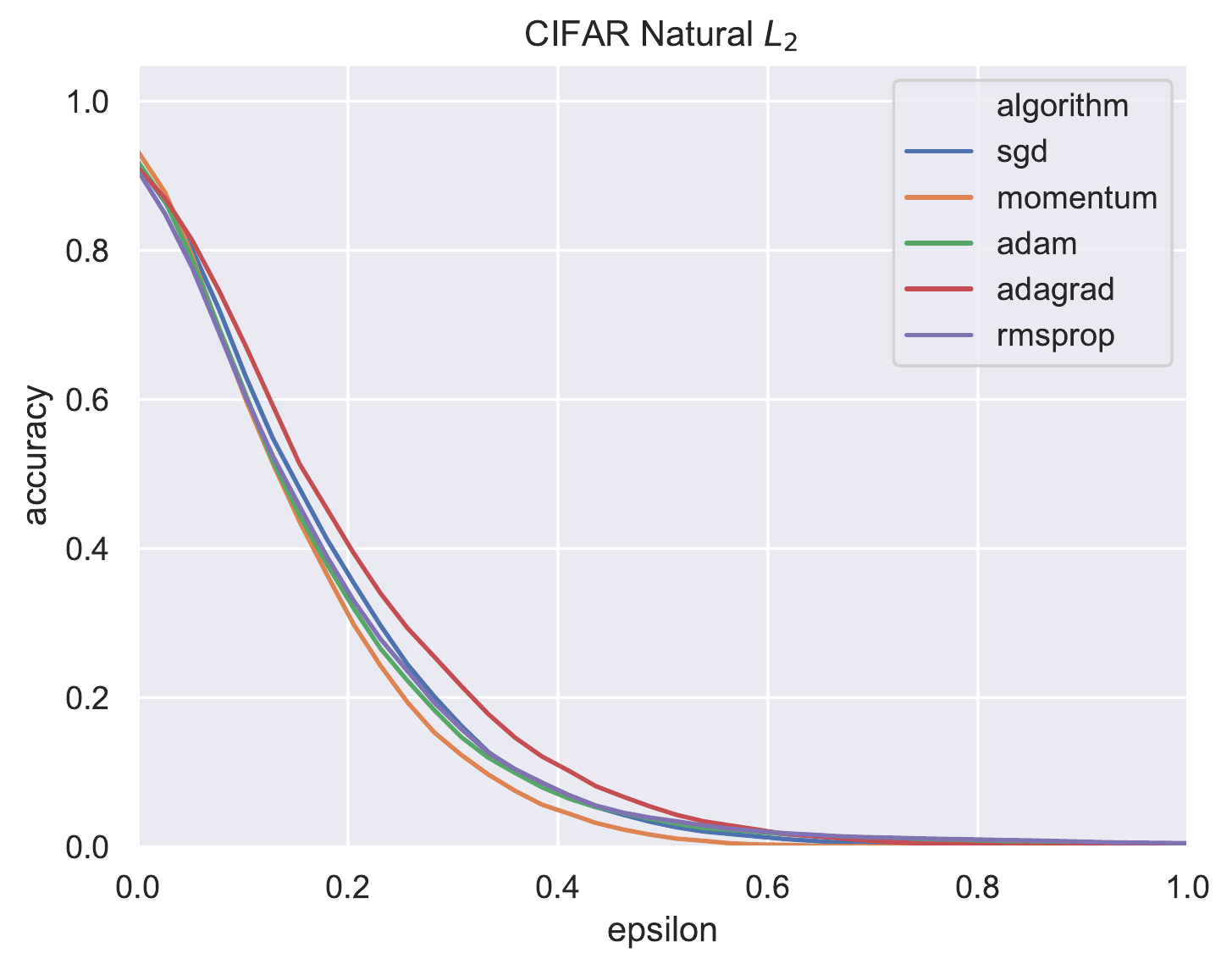}
\end{subfigure}
\begin{subfigure}{0.49\textwidth}
\includegraphics[width=0.99\linewidth]{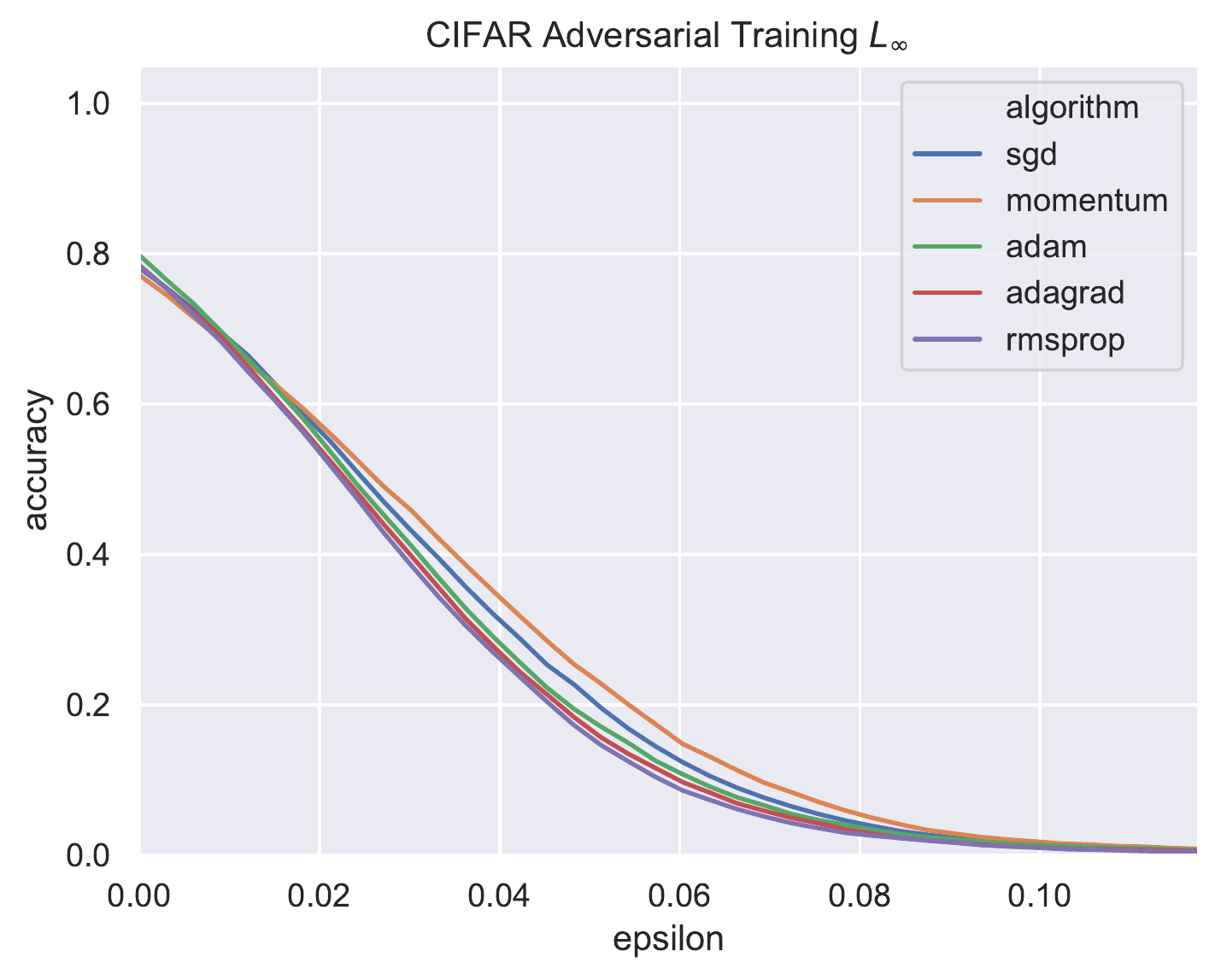}
\end{subfigure}
\begin{subfigure}{0.49\textwidth}
\includegraphics[width=0.99\linewidth]{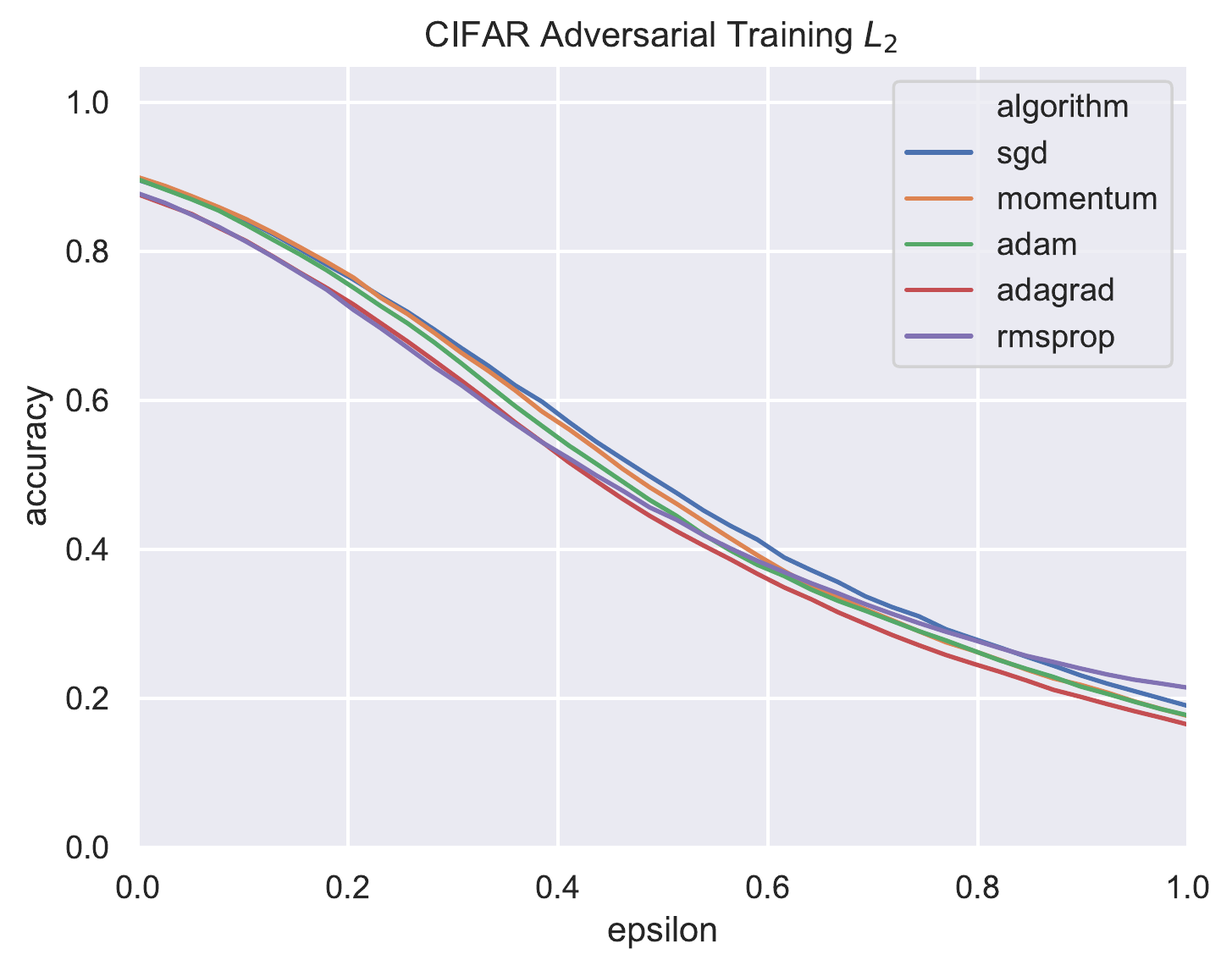}
\end{subfigure}
\caption{\textbf{Top}: Robust accuracy for naturally trained CIFAR10 models against $L_{\infty}$- and $L_{2}$-adversarial attacks. \textbf{Bottom}: Robust accuracy for adversarially trained CIFAR10 models. Left, $L_{\infty}$-adversarially trained models evaluated against $L_{\infty}$-adversarial attacks; right $L_{2}$.}
\label{fig:cifarexp}
\end{center}
\end{figure}
\newpage
\section{Conclusions}
We have presented evidence that adaptive algorithms may produce less robust solutions by giving undue influence to irrelevant features which can be exploited by an adversary. While the geometry of the loss-landscape of adversarial training in the linear regression setting is surprisingly complicated, we are still able to demonstrate the uniqueness of the optimum. In the future, it would be awesome to prove similar results for deep networks. 

\paragraph{The Geometry of $L_{\infty}$-Adversarial Training} Many of the statements that we proved about $L_{2}$-adversarial training are not true for $L_{\infty}$-adversarial training. While we believe it's possible to show that $L_{\infty}$-adversarial training gives a convex optimization problem using an argument nearly identical to ours, the geometry of the $L_{\infty}$-objective is much less strongly convex than the $L_{2}$-objective. It is unlikely that the optimal solution is unique or lies in the rowspace, for similar reasons as in Lasso regression. It would be interesting to see what, if any, statements can be made about the solutions to the $L_{\infty}$-adversarial training objective in the least-squares linear regression case. 

\paragraph{Deep Networks}
An obvious future direction is to extend our theoretical analysis to deep networks. Even the case of deep \emph{linear} networks is interesting and will likely require new techniques. Consider a deep linear function $f: x \mapsto W_l^{\top}W_{l-1}^{\top}\ldots W_{1}^{\top}x$ with $l$ trainable layers. In this case, we can solve the inner maximization problem of Equation~\ref{equ:advtrainloss} for adversarial training, which gives the loss function
\begin{equation*}
\frac{1}{2}\|XW_1W_{2}\ldots W_{l} - y\|_{2}^{2} + \epsilon\|W_1W_{2}\ldots W_{l}\|_{*}\|XW_1W_{2}\ldots W_{l} - y\|_{1} + \frac{\epsilon^2n}{2}\|W_1W_{2}\ldots W_{l}\|_{*}^2.
\end{equation*}

We know of no technique for analyzing the geometry of this loss function. Even when $\Delta$ is an $L_{2}$-ball and there are only $l=2$ trainable layers, the geometry of $\|W_1W_{2}\|_{2}$ is highly non-convex. 

We believe that the primary bottleneck for understanding the generalization and robustness properties of solutions found by optimization algorithms for deep networks is an adequate set of tools for reasoning about the geometry of high-dimensional non-convex loss landscapes .Future work should attempt to fully characterize the effect of depth on the curvature of loss landscape using tools from differential geometry. 

\section*{Acknowledgements}
The author would like to thank Jonathan Shewchuk for providing comments on earlier drafts of this manuscript and Dylan Hadfield-Menell for providing the compute resources necessary for the experiments. This work was partially funded by NSF award CCF-1909235. 
\bibliography{main}

\appendix
\section{Proofs}
\label{sec:proofs}
\subsection{Proof of Theorem~\ref{thm:adaptiverobust}}
\label{sec:adaptiverobustproof}
\begin{proof}
Let $\delta$ be an adversarial perturbation and let $x_{\operatorname{test}}$ be a test sample. Then
\begin{align*}
    \langle w_{\operatorname{ada}}, x_{\operatorname{test}} + \delta\rangle &= \langle w_{\operatorname{ada}}, x_{\operatorname{test}}\rangle  + \langle w_{\operatorname{ada}}, \delta\rangle\\
    &= \tau(y_{\operatorname{test}} + 2) + \langle w_{\operatorname{ada}}, \delta\rangle\\
    &= \tau(y_{\operatorname{test}} + 2) + \tau\left(\delta_1 + \delta_2 + \delta_3 + \sum_{i \in \mathcal{P}} \delta_i - 5\sum_{j \in \mathcal{N}} \delta_{j}\right)\\
    &= 3\tau + \tau\left(\delta_1 + \delta_2 + \delta_3 + \sum_{i \in \mathcal{P}} \delta_i - 5\sum_{j \in \mathcal{N}} \delta_{j}\right)\\
    &= \tau(y_{\operatorname{test}} + 2) - \tau \delta \left(3 + n_{+} + 5 n_{-}\right)
\end{align*}
The second last equality follows from the fact that $x_{\operatorname{test}}$ is correctly classified by $w_{\operatorname{ada}}$, and so $y_{\operatorname{test}} = 1$. Notice that to flip the sign of the classifier using the smallest $L_{\infty}$-perturbation, it is optimal to distribute the magnitude of the perturbation equally to each $\delta_i$, where the signs of each $\delta_i$ are $-1$ for $i \in \{1, 2, 3\} \cup \mathcal{P}$ and $+1$ for $i \in \mathcal{N}$. It follows that to flip the sign of the classifier requires 
\begin{align*}
    \delta \left(3 + n_{+} + 5 n_{-}\right) &> 3\\
    \delta &> \frac{3}{3 + n_{+} + 5 n_{-}}.
\end{align*}

To find the smallest $L_2$-perturbation we must instead solve the following constrained optimization problem 
\begin{equation}
\begin{aligned}
\min_{\delta}\quad & \sum_{i} \delta_{i}^2\\
\textrm{s.t.}\quad & \left(\delta_1 + \delta_2 + \delta_3 + \sum_{i \in \mathcal{P}} \delta_i - 5\sum_{j \in \mathcal{N}} \delta_{j}\right) < -3 \\
\end{aligned}
\end{equation}
where $R^2 = \sum_{i} \delta_{i}^2$ is the squared-radius of the smallest $L_2$-ball that crosses the decision boundary. The Lagrangian for this problem is 
\begin{equation*}
    \mathcal{L}(\delta, \lambda) = \sum_{i} \delta_{i}^2 + \lambda\left(\delta_1 + \delta_2 + \delta_3 + \sum_{i \in \mathcal{P}} \delta_i - 5\sum_{j \in \mathcal{N}} \delta_{j} + 3 \right)
\end{equation*}

The partial derivatives are given by 
\begin{align*}
    \frac{\partial \mathcal{L}}{\partial \delta_{i}} &= \begin{cases} 
      2\delta_{i} + \lambda & i = 1, 2, 3 \text{ or } i \in \mathcal{P} \\
      2\delta_{i} - 5 \lambda & i \in \mathcal{N}
      \end{cases}\\
    \frac{\partial \mathcal{L}}{\partial \lambda} &= \delta_1 + \delta_2 + \delta_3 + \sum_{i \in \mathcal{P}} \delta_i - 5\sum_{j \in \mathcal{N}} \delta_{j} + 3.
\end{align*}

Setting the first set of partial derivatives to $0$ gives 
\begin{equation}
\label{equ:deltasol1}
      \delta_{i} = \begin{cases} -\frac{\lambda}{2} & i = 1, 2, 3 \text{ or } i \in \mathcal{P} \\
     \frac{5\lambda}{2}  & i \in \mathcal{N}
      \end{cases},
\end{equation}
which can then be used to solve the last equation yielding
\begin{equation*}
    \lambda = \frac{6}{25n_{-} + n_{+} + 3}.
\end{equation*}
Substituting the expression for $\lambda$ back into Equation~\ref{equ:deltasol1} gives
\begin{equation}
\label{equ:deltasol2}
      \delta_{i} = \begin{cases} \frac{-3}{25n_{-} + n_{+} + 3} & i = 1, 2, 3 \text{ or } i \in \mathcal{P} \\
     \frac{15}{25n_{-} + n_{+} + 3}  & i \in \mathcal{N}
      \end{cases}.
\end{equation}
Then the minimum $L_2$-perturbation $R$ is
\begin{align*}
    R^2 &= \sum_{i} \delta_i^2\\
        &= (3 + n_{+}) \left(\frac{-3}{25n_{-} + n_{+} + 3}\right) + 5n_{-}\left(\frac{15}{25n_{-} + n_{+} + 3}\right)\\
        &= \frac{9(n_{+} +3) + 1125n_{-}}{(25n_{-}+n_{+} +3)^2}\\
    R &= \frac{\sqrt{9n_{+} + 1125n_{-} + 27}}{25n_{-}+n_{+} +3}.
\end{align*}
\end{proof}

\subsection{Proof of Theorem~\ref{thm:sgdrobust}}
\label{sec:sgdrobustproof}

It is worth taking a moment to understand $\langle w_{\text{SGD}}, x_{\operatorname{test}}\rangle$ when $y_{\operatorname{test}} = 1$ and when $y_{\operatorname{test}} = -1$. In particular, it will be important in our proofs to understand the signs of \emph{each} term. 

First, we have $\alpha_{+} > 0$ and $\alpha_{-} < 0$ by definition. When $y_{\operatorname{test}} = 1$ we have
\begin{align*}
    \langle w_{\text{SGD}}, x_{\operatorname{test}}\rangle &= (n_{+}\alpha_{+}-n_{-}\alpha_{-}) + 2(n_{+}\alpha_{+} + n_{-}\alpha_{-})\\
    &= \frac{5n_{+}+n_{-}+8n_{+}n_{-}}{15n_{+} + 3n_{-}+8n_{+}n_{-}+5} + \frac{2(5n_{+}-n_{-})}{15n_{+} + 3n_{-}+8n_{+}n_{-}+5}\\
    &= \frac{15n_{+} + 8n_{+}n_{-}-n_{-}}{15n_{+} + 3n_{-}+8n_{+}n_{-}+5}.
    %&= \frac{3n_{+}+n_{-}+8n_{+}n_{-}}{9n_{+}+3n_{-}+8n_{+}n_{-}+3} + \frac{2(3n_{+}-n_{-})}{9n_{+}+3n_{-}+8n_{+}n_{-}+3}\\
    %&= \frac{9n_{+} +8n_{+}n_{-}-n_{-}}{9n_{+}+3n_{-}+8n_{+}n_{-}+3}.
\end{align*}

The denominator is clearly positive, so $w_{\text{SGD}}$ correctly classifies $x_{\operatorname{test}}$ so long as $15n_{+} +8n_{+}n_{-}-n_{-} > 0$, which is true for any $n_{+},n_{-} \geq 1$. 

Now when $y_{\text{test}} = -1$ we have

\begin{align*}
    \langle w_{\text{SGD}}, x_{\operatorname{test}}\rangle &= -(n_{+}\alpha_{+}-n_{-}\alpha_{-}) + 2(n_{+}\alpha_{+} + n_{-}\alpha_{-})\\
    &= -\frac{5n_{+}+n_{-}+8n_{+}n_{-}}{15n_{+} + 3n_{-}+8n_{+}n_{-}+5} + \frac{2(5n_{+}-n_{-})}{15n_{+} + 3n_{-}+8n_{+}n_{-}+5}\\
    &= \frac{5n_{+} -8n_{+}n_{-}-3n_{-}}{15n_{+} + 3n_{-}+8n_{+}n_{-}+5}.
    %&= -\frac{3n_{+}+n_{-}+8n_{+}n_{-}}{9n_{+}+3n_{-}+8n_{+}n_{-}+3} + \frac{2(3n_{+}-n_{-})}{9n_{+}+3n_{-}+8n_{+}n_{-}+3}\\
    %&= \frac{3n_{+} -8n_{+}n_{-}-3n_{-}}{9n_{+}+3n_{-}+8n_{+}n_{-}+3}.
\end{align*}

In this case, $w_{\text{SGD}}$ correctly classifies $x_{\operatorname{test}}$ so long as $5n_{+} -8n_{+}n_{-}-3n_{-}< 0$, which is true for any $n_{+},n_{-} \geq 1$. Thus $w_{\text{SGD}}$ correctly classifies every test example so long as there at least one training example from each class. 

We will also be interested in the signs of the individual terms in $\langle w_{\text{SGD}}, x_{\operatorname{test}}\rangle$. Note that $5n_{+}+n_{-}+8n_{+}n_{-} > 0$ for any $n_{+},n_{-} \geq 1$, and so $(n_{+}\alpha_{+}-n_{-}\alpha_{-})$ is positive. Lastly $5n_{+}-n_{-} > 0$ so long as $n_{+} > n_{-} / 5$, and so $(n_{+}\alpha_{+} + n_{-}\alpha_{-}) > 0$ if and only if $n_{+} > n_{-} / 5$. For convenience we will assume that $n_{+} > n_{-} / 5$ from here onward which will allow us to consider fewer cases.

\begin{proof}
Let $\delta$ be an adversarial perturbation and let $x_{\operatorname{test}}$ be a test sample. Then
\begin{equation*}
\langle w_{\text{SGD}}, x_{\operatorname{test}} + \delta\rangle = \langle w_{\text{SGD}}, x_{\operatorname{test}}\rangle  + \langle w_{\text{SGD}}, \delta\rangle
\end{equation*}
where
\begin{equation*}
\langle w_{\text{SGD}}, x_{\operatorname{test}}\rangle = y_{\text{test}}(n_{+}\alpha_{+}-n_{-}\alpha_{-}) + 2(n_{+}\alpha_{+} + n_{-}\alpha_{-})
\end{equation*}
and 
\begin{equation*}
\langle w_{\text{SGD}}, \delta\rangle = (n_{+}\alpha_{+} - n_{-}\alpha_{-})\delta_1 + (n_{+}\alpha_{+} + n_{-}\alpha_{-})(\delta_2 + \delta_3) + \alpha_{+}\sum_{i \in \mathcal{P}} \delta_i + \alpha_{-}\sum_{j \in \mathcal{N}}\left(\delta_{j,1} + \ldots + \delta_{j, 5}\right).
\end{equation*}
There are two cases to consider corresponding to $y_{\operatorname{test}} = \pm 1$. 

Suppose that $y_{\operatorname{test}} = 1$. To flip the sign we need $\langle w_{\text{SGD}}, \delta\rangle < - \langle w_{\text{SGD}}, x_{\operatorname{test}}\rangle$. For brevity's sake, we will define  

\begin{equation*}
\mathcal{C} \equiv 
(n_{+}\alpha_{+} - n_{-}\alpha_{-})\delta_1 + (n_{+}\alpha_{+} + n_{-}\alpha_{-})(\delta_2 + \delta_3) + \alpha_{+}\sum_{i \in \mathcal{P}} \delta_i + \alpha_{-}\sum_{j \in \mathcal{N}}\left(\delta_{j,1} + \ldots + \delta_{j, 5}\right) + \langle w_{\text{SGD}}, x_{\operatorname{test}}\rangle.
\end{equation*}

The constraint $\mathcal{C} < 0$ is equivalent to  $\langle w_{\text{SGD}}, \delta\rangle < - \langle w_{\text{SGD}}, x_{\operatorname{test}}\rangle$. Clearly we can ensure the sign of $\langle w_{\text{SGD}}, \delta\rangle$ is negative by choosing each $\delta_i$ opposite in sign to the term by which it is multiplied in $\mathcal{C}$. Note that, by our assumptions on $n_{+}, n_{-}$, $(n_{+}\alpha_{+} - n_{-}\alpha_{-}),(n_{+}\alpha_{+} + n_{-}\alpha_{-}), \alpha_{+} > 0$ and $\alpha_{-} < 0$. Thus we choose $\sgn(\delta_{j,1\ldots,5}) = 1$ for all $j \in \mathcal{N}$ and $\sgn(\delta_i) = -1$ otherwise. Furthermore the optimal solution sets each $\delta_i$ to the same magnitude, and so to change the sign the perturbation $\delta$ must be at least

\begin{align*}
\delta &> \frac{\langle w_{\text{SGD}}, x_{\operatorname{test}}\rangle}{(n_{+}\alpha_{+} - n_{-}\alpha_{-}) + 2(n_{+}\alpha_{+} + n_{-}\alpha_{-}) + n_{+}\alpha_{+} - 5n_{-}\alpha_{-}}\\
&=\frac{\langle w_{\text{SGD}}, x_{\operatorname{test}}\rangle}{4(n_{+}\alpha_{+} - n_{-}\alpha_{-})}\\
&= \frac{3n_{+}\alpha_{+} + n_{-}\alpha_{-}}{4(n_{+}\alpha_{+} - n_{-}\alpha_{-})}\\
&=\frac{15n_{+} +8n_{+}n_{-} - n_{-}}{20n_{+} + 32n_{+}n_{-} + 4n_{-}}
\end{align*}

Now suppose that $y_{\text{test}} = -1$. In this case we need $\langle w_{\text{SGD}}, \delta\rangle > - \langle w_{\text{SGD}}, x_{\operatorname{test}}\rangle$, (equivalently $\mathcal{C} > 0$). Note that in this case $\langle w_{\text{SGD}}, x_{\operatorname{test}}\rangle$ is negative, and so we choose the signs of each $\delta_i$ to match the signs of the terms by which $\delta_i$ is multiplied. We choose $\sgn(\delta_{j,1\ldots,5}) = -1$ and $\sgn(\delta_i) = 1$ otherwise. Thus to change the sign the perturbation $\delta$ must be at least

\begin{align*}
\delta &> \frac{-\langle w_{\text{SGD}}, x_{\operatorname{test}}\rangle}{(n_{+}\alpha_{+} - n_{-}\alpha_{-}) + 2(n_{+}\alpha_{+} + n_{-}\alpha_{-}) + n_{+}\alpha_{+} - 5n_{-}\alpha_{-}}\\
&=\frac{-\langle w_{\text{SGD}}, x_{\operatorname{test}}\rangle}{4(n_{+}\alpha_{+} - n_{-}\alpha_{-})}\\
&= \frac{-n_{+}\alpha_{+} - 3n_{-}\alpha_{-}}{4(n_{+}\alpha_{+} - n_{-}\alpha_{-})}\\
&=\frac{-5n_{+} +8n_{+}n_{-} + 3n_{-}}{20n_{+} + 32n_{+}n_{-} + 4n_{-}}
\end{align*}

To find the smallest $L_2$-perturbation, in the case where $y_{\text{test}} = 1$, we must instead solve the following constrained optimization problem 

\begin{equation}
\begin{aligned}
\min_{\delta}\quad & \sum_{i} \delta_{i}^2\\
\textrm{s.t.}\quad & \mathcal{C} \leq 0 \\
\end{aligned}
\end{equation}
where $R^2 \equiv \sum_{i} \delta_{i}^2$ is the squared-radius of the smallest $L_2$-ball that crosses the decision boundary. The Lagrangian for this problem is 
\begin{equation*}
    \mathcal{L}(\delta, \lambda) = \sum_{i} \delta_{i}^2 + \lambda\mathcal{C}
\end{equation*}

The partial derivatives are given by 
\begin{align*}
    \frac{\partial \mathcal{L}}{\partial \delta_{i}} &= \begin{cases} 
      2\delta_{i} + \lambda(n_{+}\alpha_{+} - n_{-}\alpha_{-}) & i = 1 \\
      2\delta_{i} + \lambda(n_{+}\alpha_{+} + n_{-}\alpha_{-}) & i= 2, 3\\
      2 \delta_{i} + \lambda \alpha_{+} & i \in \mathcal{P}\\
      2 \delta_{i,j} + \lambda \alpha_{-} & i \in \mathcal{N}, j \in [5]\\
      \end{cases}\\
    \frac{\partial \mathcal{L}}{\partial \lambda} &= \mathcal{C}.
\end{align*}

Setting the first set of partial derivatives to $0$ gives 
\begin{equation}
\label{equ:deltasol3}
      \delta_{i} =\begin{cases} 
      -\frac{\lambda}{2}(n_{+}\alpha_{+} - n_{-}\alpha_{-}) & i = 1 \\
      -\frac{\lambda}{2}(n_{+}\alpha_{+} + n_{-}\alpha_{-}) & i= 2, 3\\
      -\frac{\lambda}{2} \alpha_{+} & i \in \mathcal{P}\\
      -\frac{\lambda}{2} \alpha_{-} & i \in \mathcal{N}, j \in [5]\\
      \end{cases}\\,
\end{equation}
which can then be used to solve the last equation yielding
\begin{equation*}
    \lambda = \frac{\langle w_{\text{SGD}}, x_{\text{test}}\rangle}{\frac{1}{2}(n_{+}\alpha_{+} - n_{-}\alpha_{-})^2 + (n_{+}\alpha_{+} + n_{-}\alpha_{-})^2 + \frac{1}{2}n_{+}\alpha_{+}^2 + \frac{5}{2}n_{-}\alpha_{-}^2}.
\end{equation*}
Substituting the expression for $\lambda$ back into Equation~\ref{equ:deltasol3} and solving for $R$ gives 
\begin{align*}
    R^2 &= \sum_{i} \delta_{i}^2\\
        &= \frac{\lambda^2}{4} \left((n_{+}\alpha_{+} - n_{-}\alpha_{-})^2 + 2(n_{+}\alpha_{+} + n_{-}\alpha_{-})^2 + n_{+}\alpha_{+}^2 + 5n_{-}\alpha_{-}^2\right)\\
        &= \frac{\lambda^2}{2} \left(\frac{1}{2}(n_{+}\alpha_{+} - n_{-}\alpha_{-})^2 + (n_{+}\alpha_{+} + n_{-}\alpha_{-})^2 + \frac{1}{2}n_{+}\alpha_{+}^2 + \frac{5}{2}n_{-}\alpha_{-}^2\right)\\
        &= \frac{\langle w_{\text{SGD}}, x_{\text{test}}\rangle^2}{(n_{+}\alpha_{+} - n_{-}\alpha_{-})^2 + 2(n_{+}\alpha_{+} + n_{-}\alpha_{-})^2 + n_{+}\alpha_{+}^2 + 5n_{-}\alpha_{-}^2}\\
    R &= \frac{\langle w_{\text{SGD}}, x_{\text{test}}\rangle}{\sqrt{(n_{+}\alpha_{+} - n_{-}\alpha_{-})^2 + 2(n_{+}\alpha_{+} + n_{-}\alpha_{-})^2 + n_{+}\alpha_{+}^2 + 5n_{-}\alpha_{-}^2}}\\
    &= \frac{3n_{+}\alpha_{+} + n_{-}\alpha_{-}}{\sqrt{(n_{+}\alpha_{+} - n_{-}\alpha_{-})^2 + 2(n_{+}\alpha_{+} + n_{-}\alpha_{-})^2 + n_{+}\alpha_{+}^2 + 5n_{-}\alpha_{-}^2}}\\
    &=\frac{15n_{+}+8n_{+}n_{-}-n_{-}}{\sqrt{64n_{+}^2n_{-}^2 + 160n_{+}^2n_{-} + 75n_{+}^2 + 32n_{+}n_{-}^2+60n_{+}n_{-}+70n_{+} + 3n_{-}^2 + 5n_{-}}}
\end{align*}

The case with $y_{\text{test}} = -1$ is similar, but with the constraint $-\mathcal{C} \leq 0$, which yields a similar solution for $\lambda$, except that the numerator is $-\langle w_{\text{SGD}}, x_{\operatorname{test}}\rangle > 0$. Subsequently 
\begin{align*}
    R &= \frac{-5n_{+}+8n_{+}n_{-}+3n_{-}}{\sqrt{64n_{+}^2n_{-}^2 + 160n_{+}^2n_{-} + 75n_{+}^2 + 32n_{+}n_{-}^2+60n_{+}n_{-}+70n_{+} + 3n_{-}^2 + 5n_{-}}}.
\end{align*}
\end{proof}

\subsection{Proof of Theorem~\ref{thm:advtrainhelps}}
The proof of Theorem~\ref{thm:advtrainhelps} is the combination of the following lemmas.

\begin{lemma}
\label{lem:optnotinnull}
Let $w \in \R^d$ be any vector and let $w_{\parallel}$ be the orthogonal projection of $w$ onto $\rowspace(X)$. Then, for the objective function
\begin{equation*}
\mathcal{L}_{2}(X, y; w) = \frac{1}{2}\|Xw - y\|_{2}^{2} + \epsilon\|w\|_{2}\|Xw - y\|_{1} + \frac{\epsilon^2n}{2}\|w\|_{2}^2.
\end{equation*}
we have that $\mathcal{L}_{2}(X, y; w) \geq \mathcal{L}_{2}(X, y; w_{\parallel})$, with equality if and only if $w = w_{\parallel}$. Hence for any optimal solution $w^{*}$ of $\mathcal{L}_{2}$, $w^{*} \in \rowspace(X)$. 
\end{lemma}
\begin{proof}
Let $w = w_{\parallel} + w_{\perp}$ be \emph{any} vector $\R^d$ where $w_{\parallel} \in \rowspace(X)$ and $w_{\perp} \in \nullspace(X)$. 

\begin{align*}
\mathcal{L}_{2}(X, y; w) &= \frac{1}{2}\|Xw - y\|_{2}^{2} + \epsilon\|w\|_{2}\|Xw - y\|_{1} + \frac{\epsilon^2n}{2}\|w\|_{2}^2\\
&= \frac{1}{2}\|X(w_{\parallel} + w_{\perp}) - y\|_{2}^{2} + \epsilon\|w_{\parallel} + w_{\perp}\|_{2}\|X(w_{\parallel} + w_{\perp}) - y\|_{1} + \frac{\epsilon^2n}{2}\|w_{\parallel} + w_{\perp}\|_{2}^2\\
&= \frac{1}{2}\|Xw_{\parallel} - y\|_{2}^{2} + \epsilon\|w_{\parallel} + w_{\perp}\|_{2}\|Xw_{\parallel} - y\|_{1} + \frac{\epsilon^2n}{2}\|w_{\parallel} + w_{\perp}\|_{2}^2\\
&= \frac{1}{2}\|Xw_{\parallel} - y\|_{2}^{2} + \epsilon\sqrt{\|w_{\parallel}\|_{2}^2 + \|w_{\perp}\|_{2}^2}\|Xw_{\parallel} - y\|_{1} + \frac{\epsilon^2n}{2}\left(\|w_{\parallel}\|_{2}^2 + \|w_{\perp}\|_{2}^2\right)\\
&\geq \frac{1}{2}\|Xw_{\parallel} - y\|_{2}^{2} + \epsilon\|w_{\parallel}\|_{2}\|Xw_{\parallel} - y\|_{1} + \frac{\epsilon^2n}{2}\|w_{\parallel}\|_{2}^2
\end{align*}
with equality if and only if $\|w_{\perp}\| = 0$. The third equality follows from the fact that $w_{\perp}$ is in $\nullspace(X)$, the fourth from the fact that $w_{\parallel} \perp w_{\perp}$. This proves the first statement. The second statement regarding $w^{*}$ follows immediately. 
\end{proof}

\begin{lemma}
\label{lem:piecewiseconvex}
Let $\mathcal{C} \in \mathcal{H}$ be a convex cell with signature $s$. The restriction of $\mathcal{L}_{2}$ to the interior of $\mathcal{C}$, denoted $\mathcal{L}_{2}|_{\Int{\mathcal{C}}}$, is a convex function. Furthermore, if $s \neq -y$ then $\mathcal{L}_{2}|_{\Int{\mathcal{C}}}$ is a strongly convex function. 

Suppose that $s = -y$, meaning that $\mathcal{C}$ contains the origin. There are four possible cases, three of which depend on the value of $\epsilon$. 

\begin{enumerate}
    \item If $Xw = y$ is an inconsistent system, then $\mathcal{L}_{2}|_{\Int{\mathcal{C}}}$ is a strongly convex function.
    \item If $Xw = y$ is a consistent system and $\epsilon \in (0,\frac{1}{\|X^{\dagger}y\|_{2}})$ then $\mathcal{L}_{2}|_{\Int{\mathcal{C}}}$ is a convex function. Specifically, $\mathcal{L}|_{\Int{\mathcal{C}}}$ is convex but not strongly convex along two line segments both of which have one endpoint at the origin and terminate at $X^{\dagger}y \pm u$ for some $u \in \nullspace(X)$ respectively. The gradient at every point on these line segments is nonzero, and so the optimal solution is found in the $\rowspace(X)$ at a point of strong convexity.
    \item If $Xw = y$ is a consistent system and $\epsilon = \frac{1}{\|X^{\dagger}y\|_{2}}$, then $\mathcal{L}_{2}|_{\Int{\mathcal{C}}}$ is a convex function. Specifically $\mathcal{L}|_{\Int{\mathcal{C}}}$ is convex but not strongly convex along a single line segment with one endpoint at the origin and the other endpoint at $X^{\dagger}y$. The optimal solution may lie along this line. 
    \item If $Xw = y$ is a consistent system and $\epsilon > \frac{1}{\|X^{\dagger}y\|_{2}}$, then $\mathcal{L}_{2}|_{\Int{\mathcal{C}}}$ is a strongly convex function.
\end{enumerate}
\end{lemma}
\begin{proof}

Let $\mathcal{C}$ be any cell in the hyperplane arrangement induced by $\|Xw-y\|_{1}$ and let $s \in {\pm 1}^{n}$ denote the signature of $\mathcal{C}$. We will show that the Hessian matrix within $\mathcal{C}$ is positive semi-definite.

The Hessian matrix $H(w)$ at a point $w \in \Int{\mathcal{C}}$ is 
\begin{equation*}
   X^{\top}X+\frac{\epsilon}{\|w\|_{2}}\left(X^{\top}sw^{\top} + ws^{\top}X\right) + \frac{\epsilon^2 n}{\|w\|_{2}^{2}} ww^{\top} - \frac{\epsilon}{\|w\|_{2}^{3}} s^{\top}(X w - y + \epsilon\|w\|_{2}s) ww^{\top} + \frac{\epsilon}{\|w\|_{2}} s^{\top}(X w - y + \epsilon\|w\|_{2}s) I.
\end{equation*}
This form of the Hessian comes from twice differentiating Equation~\ref{equ:objform1} and is equivalent to twice differentiating Equation~\ref{equ:objform2}. Note that it is crucial in the third term that the $\sgn$ function is always $\pm 1$ and not defined as $0$ when the input is $0$; this is from where the factor of $n$ is derived. At a high level, we examine the curvature induced by $H(w)$ in each unit direction $v \in \mathbb{S}^{d-1}$ at $w$ and show that it is everywhere non-negative. It is worth taking a moment to examine how each term of $H(w)$ affects the curvature of the objective at $w$. 

The term $X^{\top}X$ is a positive semi-definite matrix and induces a quadratic form with positive curvature in each eigen-direction whose corresponding eigenvalue is positive, and zero curvature in every eigen-direction corresponding to a zero eigenvalue. 

The term $\frac{1}{\|w\|_{2}}\left(X^{\top}sw^{\top} + ws^{\top}X\right)$ is a sum of outer-product matrices. Note that this matrix is symmetric, since $(X^{\top}sw^{\top})^{\top} = ws^{\top}X$. This matrix has a $(d-2)$-dimensional nullspace, corresponding to the intersection $\nullspace(w) \cap \nullspace(X^{\top}s)$. On the $2$-dimensional subspace spanned by $\{\frac{w}{\|w\|_{2}}, X^{\top}s\}$, and with respect to that basis, the outer-product has the matrix 
\begin{equation*}
    \begin{pmatrix} \frac{w}{\|w\|_{2}} \cdot X^{\top}s & \frac{w}{\|w\|_{2}} \cdot \frac{w}{\|w\|_{2}} \\ X^{\top}s \cdot X^{\top}s & \frac{w}{\|w\|_{2}} \cdot X^{\top}s \end{pmatrix}.
\end{equation*}
The eigenvalues, within this subspace, are 
\begin{equation*}
    \frac{w}{\|w\|_{2}} \cdot (X^{\top}s) \pm \sqrt{ \left( \frac{w}{\|w\|_{2}} \cdot \frac{w}{\|w\|_{2}}\right)\left( (X^{\top}s) \cdot (X^{\top}s)\right) } = \frac{w}{\|w\|_{2}} \cdot (X^{\top}s) \pm \|X^{\top}s\|_{2}.
\end{equation*}
By triangle inequality, one of these eigenvalues is always positive while the other is always negative. Thus there is one direction of positive curvature and one direction of negative curvature. The eigenvectors are
\begin{equation*}
     \frac{1}{\sqrt{1+\|X^{\top}s\|_{2}^2}} X^{\top}s \pm \frac{\|X^{\top}s\|_{2}}{\sqrt{1+\|X^{\top}s\|_{2}^2}} \frac{w}{\|w\|_{2}}.
\end{equation*}

The term $\frac{\epsilon^2 n}{\|w\|_{2}^{2}} ww^{\top}$ induces positive curvature in the direction $w$ with eigenvalue $\epsilon^2 n$ and $0$ curvature in every direction orthogonal to $w$. 

The term $-\frac{\epsilon}{\|w\|_{2}^{3}} s^{\top}(X w - y + \epsilon\|w\|_{2}s) ww^{\top}$ induces negative curvature in the direction $w$ with eigenvalue $-\frac{\epsilon}{\|w\|_{2}} s^{\top}(X w - y + \epsilon\|w\|_{2}s)$. However the negative curvature in the direction $w$ is exactly undone by the positive curvature induced by the term $\frac{\epsilon}{\|w\|_{2}} s^{\top}(X w - y + \epsilon\|w\|_{2}s) I$ which induces positive curvature in every direction with eigenvalues all equal to $\frac{\epsilon}{\|w\|_{2}} s^{\top}(X w - y + \epsilon\|w\|_{2}s)$. The result of the sum of these two terms is a quadratic form which induces $0$ curvature in the direction $w$ and positive curvature in every direction orthogonal to $w$ with eigenvalue $\frac{\epsilon}{\|w\|_{2}} s^{\top}(X w - y + \epsilon\|w\|_{2}s)$. Note that the value $\frac{\epsilon}{\|w\|_{2}} s^{\top}(X w - y + \epsilon\|w\|_{2}s)$ is positive by definition, since $w$ is in the convex cell with signature $s$, and so 
\begin{equation*}
\frac{\epsilon}{\|w\|_{2}} s^{\top}(X w - y + \epsilon\|w\|_{2}s) = \frac{\epsilon}{\|w\|_{2}} (\|Xw-y\|_{1} + \epsilon \|w\|_{2}n) > 0.
\end{equation*}

Let $v \in \mathbb{S}^{d-1}$ be a unit vector. The curvature in the direction $v$ is proportional (with positive constant of proportionality) to 
\begin{align*}
    v^{\top} H(w) v &= v^{\top}X^{\top}Xv+\frac{\epsilon}{\|w\|_{2}}\left(v^{\top}(X^{\top}sw^{\top} + ws^{\top}X)v\right) + \frac{\epsilon^2 n}{\|w\|_{2}^{2}} v^{\top}ww^{\top}v \\
    &- \frac{\epsilon}{\|w\|_{2}^{3}} s^{\top}(X w - y + \epsilon\|w\|_{2}s) v^{\top}ww^{\top}v + \frac{\epsilon}{\|w\|_{2}} s^{\top}(X w - y + \epsilon\|w\|_{2}s) v^{\top}v\\
    &= \|Xv\|_{2}^2 + \frac{2\epsilon}{\|w\|_{2}} (w^{\top}v)(s^{\top}Xv) + \frac{\epsilon^2n}{\|w\|_{2}^2}(w^{\top}v)^2 + \epsilon\left(\frac{\|Xw-y\|_{1}}{\|w\|_{2}} + \epsilon n\right)\left(1 - \left(\frac{w}{\|w\|_{2}}\cdot v\right)^2\right) \\
    &= \underbrace{\|Xv\|_{2}^2 + \frac{2\epsilon\sqrt{n}}{\|w\|_{2}} (w^{\top}v)\|Xv\|_{2} \cos{\varphi} + \frac{\epsilon^2n}{\|w\|_{2}^2}(w^{\top}v)^2}_{\text{term 1}} + \underbrace{\epsilon\left(\frac{\|Xw-y\|_{1}}{\|w\|_{2}} + \epsilon n\right)\left(1 - \cos^2{\theta}\right)}_{\text{term 2}}
\end{align*}
where $\varphi = \angle(s, Xv)$ and $\theta = \angle(w, v)$. It's easy to see that term 2 is always greater than or equal to $0$, since $\cos^2{\theta} \in [0,1]$, with equality when $\cos^2{\theta} = 1$. By the quadratic formula, term 1 is also always greater than or equal to $0$, with equality when $\cos{\varphi} = \pm 1$ \emph{and} $\sgn(\cos{\varphi}) \neq \sgn(w^{\top}v)$; otherwise the zeros given by the quadratic formula have an imaginary component that depends on $\sin{\varphi}$. Thus, at this point, we see that $H(w)$ is at least positive semi-definite in $\Int{\mathcal{C}}$. 

We wish to derive under which conditions this inequality is strict, implying that $H(w)$ is positive-definite in $\mathcal{C}$. First we will show that if $w$ is in a cell of the hyperplane arrangement whose signature is $s \neq -y$, then $H(w)$ is positive definite. The conditions which must be true for $v^{\top}H(w)v = 0$ imply that $s = -y$; $w$ must be in the cell that contains the origin.

For $v^{\top}H(w)v = 0$ we need \emph{both} term 1 \emph{and} term 2 to be equal to $0$. Term 2 is equal to $0$ if and only if $\cos{\theta} = \pm 1$, which implies that $v \parallel w$. Since $v$ is a unit vector, we have $v = \pm\frac{w}{\|w\|_{2}}$. For term 1 to be equal to $0$ we need $\cos{\varphi} = \pm 1$ and $\sgn(\cos{\varphi}) \neq \sgn(w^{\top}v)$. The first of these two conditions implies that $s \parallel Xv$. Suppose that $v = \frac{w}{\|w\|_{2}}$; then $Xv = -\alpha s$ for some $\alpha > 0$. So we have that 
\begin{align*}
    \frac{Xv}{\|Xv\|_{2}} &= -\frac{s}{\|s\|_{2}}\\
    \frac{Xw}{\|Xw\|_{2}} &=\\
    Xw &= -\frac{\|Xw\|_{2}}{\|s\|_{2}}s \\
    x_{i}^{\top}w &= -\frac{\|Xw\|_{2}}{\|s\|_{2}}s_{i}
\end{align*}

Now, $w \in \mathcal{C}$, which implies that $\sgn(x_i^{\top}w - y_i) = s_i$. If $s_i = 1$, then 
\begin{align*}
    x_i^{\top}w - y_i &> 0\\
    x_i^{\top}w &> y_i\\
    -\frac{\|Xw\|_{2}}{\|s\|_{2}}s_{i} &> y_i\\
    0 > -\frac{\|Xw\|_{2}}{\|s\|_{2}}s_{i} &> y_i
\end{align*}
which implies that $y_i = -1$. The case where $s_i = -1$ is similar, as is the case where $v = -\frac{w}{\|w\|_{2}}$. All together, we have that $s = -y$.

Thus the necessary (\emph{not} sufficient) conditions for $v^{\top}H(w)v = 0$ can only be satisfied if $s = -y$. If $s \neq -y$ then $H(w)$ is defined and positive-definite everywhere in $\Int{\mathcal{C}}$. 

We now turn our attention toward a necessary condition, which when combined with our other necessary conditions, give a set of sufficient conditions for $H(w)$ to be positive semi-definite but not positive-definite. Suppose that $\cos{\theta} = \pm 1, \cos{\varphi} = \pm 1$ and $\sgn(\cos{\varphi}) \neq \sgn(w^{\top}v)$. By the above discussion $s = -y$. Suppose $v = \frac{w}{\|w\|_{2}}$ and, thus, $\cos{\varphi} = -1$. Under these conditions we have that 

\begin{align*}
    v^{\top} H(w) v &= \|Xv\|_{2}^2 - \frac{2\epsilon\sqrt{n}}{\|w\|_{2}} (w^{\top}v)\|Xv\|_{2} + \frac{\epsilon^2n}{\|w\|_{2}^2}(w^{\top}v)^2 + \epsilon\left(\frac{\|Xw-y\|_{1}}{\|w\|_{2}} + \epsilon n\right)\left(1 - \cos^2{\theta}\right)\\
    &= \left(\|Xv\|_{2} - \frac{\epsilon\sqrt{n}}{\|w\|_{2}} w^{\top} v\right)^2\\
    &=  \left(\|Xv\|_{2} - \epsilon\sqrt{n}\right)^2.
\end{align*}

Note that when any one of the conditions detailed in the previous paragraph do not hold, the first equality is instead a lower bound on $ v^{\top} H(w) v$. From this we see that the final necessary condition for $v^{\top} H(w) v = 0$ is for $\|Xv\|_{2} = \epsilon \sqrt{n}$. Since $\cos{\varphi} = -1$, we must have $Xv = -\epsilon s = \epsilon y$ which implies $v = \epsilon X^{\dagger}y + u$ for $u \in \nullspace(X)$. Recall that $v$ is a unit vector, so $\|v\|^2_{2} = \|\epsilon X^{\dagger}y + u\|^2_{2} =\epsilon^2\|X^{\dagger}y\|_{2}^2 + \|u\|^2_{2} =  1$, from which it follows that $\epsilon = \frac{\sqrt{1 - \|u\|^2_{2}}}{\|X^{\dagger}y\|_{2}}$. 

The relationship $\epsilon = \frac{\sqrt{1 - \|u\|^2_{2}}}{\|X^{\dagger}y\|_{2}}$ gives three intervals for $\epsilon$ in which the curvature of $\mathcal{L}_{2}$ behaves qualitatively differently. For $\epsilon \in (0, 1/\|X^{\dagger}y\|_{2})$ the equation has two solutions $\pm u$ in the $\nullspace(X)$ with $\|u\|_{2} < 1$.  Since $w \parallel v$ this ray of $0$ curvature lies outside of $\rowspace(X)$, and, by Lemma~\ref{lem:optnotinnull}, the gradient cannot be $0$ along this ray. For $\epsilon = 1/\|X^{\dagger}y\|_{2}$, the solution is given by $u = 0$ and so there is a single ray in the direction of $X^{\dagger}y$ in the $\rowspace(X)$. This ray is parameterized by $\alpha X^{\dagger}y$ for $\alpha \in (0, 1)$. The gradient may or may not be zero along this ray. Finally for $\epsilon > 1/\|X^{\dagger}y\|_{2}$ there is no solution to the relationship and $\mathcal{L}_{2}$ is strongly convex within $\mathcal{C}$ with signature $s = -y$. 

Before concluding we must address the fact that the Hessian $H$ is not defined at $w = 0$. Let $\{O_{i}\}_i$ be the set of $2^d$ closed orthants of $\R^d$. We further subdivide $\mathcal{C}$ with signature $s = -y$ as $\mathcal{C}_i = \mathcal{C} \cap O_{i}$. Within the relative interiors of each $\mathcal{C}_{i}$, $\mathcal{L}_{2}|_{\Int{C_i}}$ is twice differentiable everywhere with Hessian as described above. Thus $\mathcal{L}_{2}|_{\Int{C_i}}$ is convex for all $i$. 

Let $w \in \Bd{\mathcal{C}_{i}} \cap \Int{\mathcal{C}}$ and $w \neq 0$. Then the subdifferential $\partial \mathcal{L}_{2}|_{C_i}(w)$ is nonempty and, in particular, contains the gradient $\nabla \mathcal{L}_{2}(w)$, which is defined at $w$ since $w \in \Int{\mathcal{C}}$ and $w \neq 0$. The intersection $\partial \mathcal{L}_{2}|_{C_i}(w) \cap \partial \mathcal{L}_{2}|_{C_j}(w) = \{\nabla \mathcal{L}_{2}(w)\}$ for $w \in \Bd{\mathcal{C}_{i}} \cap \Bd{\mathcal{C}_{j}} \cap \Int{\mathcal{C}}$, since $\mathcal{L}_{2}$ is actually differentiable at $w$. 

Now let $w \in \Int{\mathcal{C}_{i}}$ and $w' \in \Int{\mathcal{C}_{j}}$ for $i \neq j$ and such that the line segment $ww'$ does not intersect the origin in its relative interior. Further suppose that $\mathcal{C}_{i}$ and $\mathcal{C}_{j}$ are adjacent along the line segment $ww'$, meaning that there is a single point $w''$ at which  the line segment $ww'$ leaves $\Int{\mathcal{C}_{i}}$ and enters $\Int{\mathcal{C}_{j}}$. Note that $w, w'', w'$ are collinear. Then
\begin{align*}
    \langle \nabla \mathcal{L}_{2}(w), w' - w\rangle &=  \langle \nabla \mathcal{L}_{2}(w), w' - w''\rangle +  \langle \nabla \mathcal{L}_{2}(w), w'' - w\rangle\\
    &= \frac{\|w'' - w\|_{2}}{\|w' - w''\|_{2}}\langle \nabla \mathcal{L}_{2}(w), w'' - w\rangle +  \langle \nabla \mathcal{L}_{2}(w), w'' - w\rangle\\
    &\leq \frac{\|w'' - w\|_{2}}{\|w' - w''\|_{2}}\langle \nabla \mathcal{L}_{2}(w''), w'' - w\rangle + \langle \nabla \mathcal{L}_{2}(w), w'' - w\rangle\\
    &= \langle \nabla \mathcal{L}_{2}(w''), w' - w''\rangle +  \langle \nabla \mathcal{L}_{2}(w), w'' - w\rangle\\
    &\leq \mathcal{L}_{2}|_{\Int{\mathcal{C}_j}}(w') - \mathcal{L}_{2}|_{\Int{\mathcal{C}_j}}(w'') + \mathcal{L}_{2}|_{\Int{\mathcal{C}_i}}(w'') - \mathcal{L}_{2}|_{\Int{\mathcal{C}_i}}(w)\\
    &= \mathcal{L}_{2}|_{\Int{\mathcal{C}_j}}(w') - \mathcal{L}_{2}|_{\Int{\mathcal{C}_i}}(w)\\
    &= \mathcal{L}_{2}(w') - \mathcal{L}_{2}(w).
\end{align*}
The second equality follows from collinearity and the first inequality follows from convexity of  $\mathcal{L}_{2}|_{\Int{\mathcal{C}_i}}$ from which we can derive $\langle \nabla \mathcal{L}_{2}(w'')- \nabla \mathcal{L}_{2}(w) , w'' - w\rangle \geq 0$. The remaining steps are straightforward. This argument can be extended to a line segment $ww'$ for $w$ and $w'$ in two cells that only intersect at the origin in a straightforward manner using induction. Thus it follows that $\mathcal{L}_{2}|_{\Int{C}}$ is convex along $ww'$. All that remains is the case where $ww'$ intersects the origin. 

Suppose that $ww'$, parameterized by $\ell(t) = (1-t)w + tw'$ for $t \in [0,1]$, intersects the origin. Choose any unit vector $v$ such that $v$ is not parallel to  $ww'$. Then consider the perturbed line segment $\tilde{\ell}(t, \epsilon) = (1-t)(w+\epsilon v) + t(w' + \epsilon v) = \ell(t) + \epsilon v$ for $\epsilon > 0$. Let $t_0$ be such that $\ell(t_0) = 0$. As $\epsilon \rightarrow 0$, $\tilde{\ell}(t, \epsilon) \rightarrow \ell(t)$ and, in particular, $\tilde{\ell}(t_0, \epsilon) \rightarrow 0$. Since $v$ is not parallel with $ww'$, $\tilde{\ell}(t, \epsilon)$ does not intersect the origin for $\epsilon > 0$, and so $\mathcal{L}_{2}|_{\Int{\mathcal{C}}}(\tilde{\ell}(t, \epsilon)) \leq (1-t) \mathcal{L}_{2}|_{\Int{\mathcal{C}}}(\tilde{\ell}(0, \epsilon)) + t \mathcal{L}_{2}|_{\Int{\mathcal{C}}}(\tilde{\ell}(1, \epsilon))$. Taking $\epsilon \rightarrow 0$ convexity follows from continuity of $\mathcal{L}_{2}|_{\Int{\mathcal{C}}}$. Note that this approach only applies when $d\geq 2$; however the $d=1$ for $\mathcal{L}_{2}$ case is straightforward. 

\end{proof}

\begin{lemma}
\label{lem:strictconvex}
$\mathcal{L}_{2}$ is a convex function. If $\mathcal{L}_{2}|_{\Int{\mathcal{C}}}$ for $\mathcal{C}$ with signature $s = -y$ is a strongly convex function, then $\mathcal{L}_{2}$ is a strictly convex function. Furthermore transitions between two cells are strictly convex. 
\end{lemma}

\begin{proof}
Let $w, w' \in \R^d$ be any two points. The line segment $ww'$ with endpoints $w$ and $w'$ is parameterized by $w_t = (1-t)w + tw'$ for $t\in [0, 1]$. If $ww' \subset \Int{\mathcal{C}}$ for some $\mathcal{C}$ then Lemma~\ref{lem:piecewiseconvex} gives the results. Suppose that $w \in \mathcal{C}$ and $w' \in \mathcal{C}'$ are in distinct cells of the hyperplane arrangement and that $ww'$ intersect the boundaries of these cells at $t_1, \ldots, t_m$. This partitions the interval $[0,1]$ into $m+1$ subintervals $[0, t_1] \cup [t_1, t_2] \cup \ldots \cup [t_m, 1]$, in each of which the function $\mathcal{L}_{2}$ is convex along $w_{t_i}w_{t_{i+1}}$ by Lemma~\ref{lem:piecewiseconvex}. 

Consider the base case where $m = 1$. The point $w_1 \in \mathcal{C} \cap \mathcal{C}'$, where the line segment $ww'$ leaves $\mathcal{C}$ and enters $\mathcal{C}'$. The facet $f  = \mathcal{C} \cap \mathcal{C}'$ is a $(d-k)$-dimensional facet, where $k$ is the number of hyperplanes that intersect at $w_1$. Said differently, at $w_1$ the signs of $k$ hyperplane equations $x_i^{\top}w - y_i$ flip. 

Imagine removing these $k$ hyperplanes, then $w$ and $w'$ lie in the same cell of the induced hyperplane arrangement, and, by Lemma~\ref{lem:piecewiseconvex}, the objective function $\mathcal{L}_{2}^{(-k)}$ with these $k$ hyperplanes removed is convex. (Simply repeat the argument for $n - k$ samples.) Thus we have 

\begin{align*}
    \mathcal{L}_{2}(w_{t_1}) &= \mathcal{L}_{2}^{(-k)}(w_{t_1}) + \frac{1}{2} \sum_{i} \left(\langle x_i, w_{t_1}\rangle - y_i + \epsilon \sgn(\langle x_i, w_{t_1}\rangle - y_i)\|w_{t_1}\|_{2}\right)^2\\
    &= \mathcal{L}_{2}^{(-k)}(w_{t_1}) + \frac{1}{2} \sum_{i} \left(\epsilon \sgn(\langle x_i, w_{t_1}\rangle - y_i)\|w_{t_1}\|_{2}\right)^2\\
    &= \mathcal{L}_{2}^{(-k)}(w_{t_1}) + \frac{1}{2}\sum_{i} \epsilon^2\|w_{t_1}\|_{2}^2\\
    &\leq (1-t_1)\mathcal{L}_{2}^{(-k)}(w) + t_1 \mathcal{L}_{2}^{(-k)}(w') + (1-t_1)\frac{1}{2}\sum_{i} \epsilon^2\|w\|_{2}^2 + t_1\frac{1}{2}\sum_{i} \epsilon^2\|w'\|_{2}^2\\
    &< (1-t_1)\mathcal{L}_{2}^{(-k)}(w) + t_1 \mathcal{L}_{2}^{(-k)}(w') \\
    &\qquad\qquad+ (1-t_1)\frac{1}{2}\sum_{i}\left(\left(\langle x_i, w\rangle - y_i\right)^2 + 2 \epsilon \|w\|_{2}\sgn(\langle x_i, w\rangle - y_i)\left(\langle x_i, w\rangle - y_i\right) + \epsilon^2\|w\|_{2}^2\right)\\ 
    &\qquad\qquad + t_1\frac{1}{2}\sum_{i} \left(\left(\langle x_i, w'\rangle - y_i\right)^2 + 2 \epsilon\|w'\|_{2}\sgn(\langle x_i, w'\rangle - y_i)\left(\langle x_i, w'\rangle - y_i\right) + \epsilon^2\|w'\|_{2}^2\right)\\
    &= (1-t_1)\mathcal{L}(w) + t_1\mathcal{L}(w')
\end{align*}

The second equality follows from the crucial fact that, at $w_{t_1}$, each hyperplane constraint $x_i^{\top}w - y_i = 0$. The first inequality follows from the convexity of $\mathcal{L}_{2}^{(-k)}$ and $\|w\|_{2}$. The second inequality follows from adding strictly positive terms. The final equality follows by definition. With this fact we are ready to show the convexity of $\mathcal{L}_{2}$ along the entire segment $ww'$.

\begin{align*}
    \mathcal{L}_{2}(w_{t}) &\leq \begin{cases} 
      (1-\alpha(t)) \mathcal{L}_{2}(w) + \alpha(t)\mathcal{L}_{2}(w_{t_1}) & t \in [0, t_1] \\
       (1-\beta(t)) \mathcal{L}_{2}(w_{t_1}) + \beta(t) \mathcal{L}_{2}(w')& t\in [t_1, 1]
   \end{cases}\\
   &<\begin{cases} 
      (1-\alpha(t)) \mathcal{L}_{2}(w) + \alpha(t) \left((1-t_1)\mathcal{L}_{2}(w) + t_1\mathcal{L}_{2}(w')\right) & t \in [0, t_1] \\
       (1-\beta(t)) \left((1-t_1)\mathcal{L}_{2}(w) + t_1\mathcal{L}_{2}(w')\right) +  \beta(t) \mathcal{L}_{2}(w')& t\in [t_1, 1]
   \end{cases}\\
   &= \begin{cases} 
      (1-t)\mathcal{L}_{2}(w) + t\mathcal{L}_{2}(w') & t \in [0, t_1] \\
       (1-t)\mathcal{L}_{2}(w) + t\mathcal{L}_{2}(w')& t\in [t_1, 1]
   \end{cases}\\
   &=  (1-t)\mathcal{L}_{2}(w) + t\mathcal{L}_{2}(w').
\end{align*}

The first inequality follows from the fact that $\mathcal{L}_{2}$ is convex along each sub-segment. The functions $\alpha: [0,t_1] \rightarrow [0,1], \beta: [t_1, 1] \rightarrow [0,1]$ are the reparameterization functions defined as $\alpha(t) = \frac{t}{t_1}, \beta(t) = \frac{t-t_1}{1-t_1}$. The second inequality follows from the statement we proved about $\mathcal{L}_{2}(w_{t_1})$. The final equality follows from the definitions of $\alpha, \beta$. Thus $\mathcal{L}_{2}$ is convex along the line segment $ww'$ when $m = 1$. 

Repeating the argument inductively gives that $\mathcal{L}$ is convex along $ww'$ for any $m$.  We have proven that the transitions between cells are strictly convex. When $\mathcal{L}_{2}$ restricted to each cell $\mathcal{C}$ is strongly convex, then the whole function $\mathcal{L}_{2}$ is strictly convex, otherwise $\mathcal{L}_{2}$ is convex. 
\end{proof}

\begin{lemma}
\label{lem:subdiff}
$\mathcal{L}_{2}$ is subdifferentiable everywhere. Let $w \in \R^d$. If $w \in \Int{\mathcal{C}}$ for some $\mathcal{C} \in \mathcal{H}$ and $w \neq 0$, then $\mathcal{L}_{2}$ is differentiable at $w$ with
\begin{equation}
\label{equ:grad}
    \nabla \mathcal{L}_{2}(w) = X^{\top}(Xw - y) + \epsilon \|w\|_{2} X^{\top}s + \epsilon \|Xw - y\|_{1} \frac{w}{\|w\|_{2}} + \epsilon^2 n w.
\end{equation}

If $w = 0$, then the subdifferential $\partial \mathcal{L}_{2}(0)$ is parameterized by replacing $\frac{w}{\|w\|_{2}}$ in Equation~\ref{equ:grad} with any $g$ such that $\|g\|_{2} \leq 1$. 

Otherwise $w \in \Bd{\mathcal{C}}$, meaning that $w \in f$ for some $(d-k)$-dimensional face $f$ of $\mathcal{C}$. Let $\{i_1, \ldots, i_k\} \subset [n]$ be the $k$ indices for which $w \in h_{i_j}$ ($x_{i_j}^{\top}w - y_{i_{j}} = 0$). The subdifferential $\partial \mathcal{L}_{2}(w)$ is 
non-empty and is parameterized by every setting of $s_{i_j} \in [-1,1]$ in Equation~\ref{equ:grad}.
\end{lemma}

\begin{proof}
By Lemma~\ref{lem:strictconvex}, the epigraph of $\mathcal{L}_{2}$ is a convex set. The Separating Hyperplane Theorem implies the existence of a supporting hyperplane at every point $(w, \mathcal{L}_{2}(w))$. If $w \in \Int{\mathcal{C}}$ for some cell $\mathcal{C}$ in the hyperplane arrangement, then $\mathcal{L}_{2}$ is differentiable at $w$ and there is a single supporting hyperplane at $(w, \mathcal{L}_{2}(w))$. Otherwise, $w$ is on the boundary $\partial \mathcal{C}$ of some $\mathcal{C}$, and the existence of a supporting hyperplane implies the existence of a subgradient of $\mathcal{L}_{2}$ at $w$.

The gradient of $\mathcal{L}_{2}$, where defined, is 
\begin{equation*}
    \nabla \mathcal{L}_{2}(w) = X^{\top}(Xw - y) + \epsilon \|w\|_{2} X^{\top}s + \epsilon \|Xw - y\|_{1} \frac{w}{\|w\|_{2}} + \epsilon^2 n w.
\end{equation*}

Suppose that $w = 0$. Let $v \in S^{d-1}$ be a unit vector and $\delta > 0$ sufficiently small. Then convexity of $\mathcal{L}_{2}|_{\Int{C}}$ for $\mathcal{C}$ with signature $s = -y$ (Lemma~\ref{lem:piecewiseconvex}) and a standard limit argument gives
\begin{align*}
    \langle -X^{\top}y + \epsilon \|y\|_{1}v, w'\rangle &= \lim_{\delta \rightarrow 0^{+}} \langle \nabla \mathcal{L}_{2}(\delta v), w' - \delta v\rangle\\
    &\leq \lim_{\delta \rightarrow 0^+} \mathcal{L}_{2}(w') - \mathcal{L}_{2}(0)\\
    &= \mathcal{L}_{2}(w') - \mathcal{L}_{2}(0)
\end{align*}
where the inequality holds for all $\delta > 0$ by continuity. Thus $v$ induces a subgradient at $w = 0$. 

Let $g$ be a vector such that $\|g\|_{2} \leq 1$. $g$ can be written as $g = (1 - \alpha)v + \alpha(-v)$ for $\alpha = (1-\|g\|_{2})/2$ for subgradients $v, -v \in \partial \mathcal{L}_{2}(0)$. Since $\|g\|_{2} \leq 1$, $\alpha \in [0, 1]$. So
\begin{align*}
    \langle -X^{\top}y + \epsilon \|y\|_{1}g, w'\rangle &=  \langle -X^{\top}y + \epsilon \|y\|_{1}\left((1 - \alpha)v + \alpha(-v)\right), w'\rangle\\
    &= (1-\alpha)\left(\langle -X^{\top}y +  \epsilon \|y\|_{1}v, w' \rangle\right) + \alpha\left(\langle -X^{\top}y + \epsilon \|y\|_{1}(-v), w'\rangle\right)\\
    &\leq (1-\alpha)(\mathcal{L}_{2}(w') - \mathcal{L}_{2}(0)) + \alpha(\mathcal{L}_{2}(w') - \mathcal{L}_{2}(0))\\
    &= \mathcal{L}_{2}(w') - \mathcal{L}_{2}(0),
\end{align*}
and so $g$ induces a subgradient at $w = 0$ as well.

To find the subdifferential $\partial \mathcal{L}_{2}(w)$ for $w \neq 0$, we consider $w \in f$ for some $(d-k)$-dimensional facet $f$ of $\mathcal{C}$, and proceed by induction over $k$.

In the base case, $k = 1$. Since $f$ is $(d-1)$-dimensional, there is only one tight hyerplane equation $x_i^{\top}w - y_i = 0$ at $w$. Let $h = \{w \in \R^d: x_i^{\top}w - y_i = 0\}$ denote the hyperplane and let $h^{+}, h^{-}$ denote the halfspaces in which $\sgn(x_i^{\top}w - y_i) = \pm 1$ respectively. 
The limit of the gradient as approach $w$ by a sequence in $h^{+}$ is $\nabla \mathcal{L}_{2}(w)$ where $s_i = 1$; similarly approaching $w$ by a sequence in $h^{-}$ gives $\nabla \mathcal{L}_{2}(w)$ where $s_i = -1$. These vectors define two supporting hyperplanes of the epigraph at $(w, \mathcal{L}_{2}(w))$. 

Note that only $\epsilon \|w\|_{2} X^{\top}s$ and  $\epsilon \|Xw - y\|_{1} \frac{w}{\|w\|_{2}}$ in $\nabla \mathcal{L}_{2}$ depend upon $s$, and when $x_i^{\top}w - y_i = 0$, $\|Xw - y\|_{1}$ is identical regardless of the setting of $s_i$, so we need only consider $\epsilon \|w\|_{2} X^{\top}s$. Let $s_i \in [-1, 1]$, then
\begin{align*}
    \epsilon \|w\|_{2} \langle X^{\top}s, w' - w \rangle &= \epsilon \|w\|_{2} \langle \left(\sum_{j \neq i} s_{j} x_{j}\right) + s_i x_i, w' - w \rangle\\
    &= \epsilon \|w\|_{2} \left(\langle \sum_{j \neq i} s_{j} x_{j}, w' -w \rangle + \langle s_i x_i, w' - w \rangle\right)\\
    &= \epsilon \|w\|_{2} \left(\langle \sum_{j \neq i} s_{j} x_{j}, w' -w \rangle + (1- \alpha) \langle -x_i, w' - w \rangle + \alpha \langle x_i, w' - w \rangle\right)
\end{align*}
where $\alpha = \frac{1 + s_i}{2}$. Then we can express $\nabla \mathcal{L}_{2}(w)|_{s_i}$, where $s_i \in [-1,1]$, as a convex combination of the terms  $\mathcal{L}_{2}(w)|_{s_i = -1}, \mathcal{L}_{2}(w)|_{s_i=1}$. 

\begin{align*}
    \langle \nabla \mathcal{L}_{2}(w)|_{s_i}, w' - w\rangle &=  \langle X^{\top}(Xw - y) + \epsilon \|w\|_{2} X^{\top}s + \epsilon \|Xw - y\|_{1} \frac{w}{\|w\|_{2}} + \epsilon^2 n w, w' - w \rangle\\
    &= \langle X^{\top}(Xw - y) + \epsilon \|Xw - y\|_{1} \frac{w}{\|w\|_{2}} + \epsilon^2 n w, w' - w \rangle\\ &\quad+ \epsilon \|w\|_{2} \left(\langle \sum_{j \neq i} s_{j} x_{j}, w' -w \rangle + (1- \alpha) \langle -x_i, w' - w \rangle + \alpha \langle x_i, w' - w \rangle \right) \\
    &= (1-\alpha)\langle \nabla \mathcal{L}_{2}(w)|_{s_i=-1}, w' - w\rangle + \alpha \langle \nabla \mathcal{L}_{2}(w)|_{s_i=1}, w' - w\rangle\\
    &\leq (1-\alpha) \left(\mathcal{L}_{2}(w') - \mathcal{L}_{2}(w)\right) + \alpha \left(\mathcal{L}_{2}(w') - \mathcal{L}_{2}(w)\right)\\
    &= \mathcal{L}_{2}(w') - \mathcal{L}_{2}(w)
\end{align*}
where the inequality follows from the fact that $\nabla \mathcal{L}_{2}(w)|_{s_i=-1}, \nabla \mathcal{L}_{2}(w)|_{s_i=1}$ are subgradients. Thus $\mathcal{L}_{2}(w)|_{s_i}$ is a subgradient for any $s_i \in [-1,1]$ at $w$. 

Now suppose that $w \in f$ is a $(d-k)$-dimensional facet and the statement holds for all $1 \leq j < k$. Let $\{i_1, \ldots, i_k\}$ index the hyperplane equations $x_{i_j}^{\top}w - y_{i_j} = 0$ at $w$. Consider the subset of hyperplane equations $\{i_1, \ldots, i_{k-1}\}$ along which subgradients exist for any setting of $s_{i_j} \in [-1, 1]$ by the inductive hypothesis. An identical limit argument as above implies the existence of two subgradients at $w$ with $s_{i_k} = \pm 1$. Then an identical calculation to those above imply that $\nabla \mathcal{L}_{2}(w)|_{s_{i_k}}$ is a subgradient for any $s_{i_k} \in [-1,1]$. Thus at $w \in f$ there exists a subdifferential parameterized by $s_{i_j} \in [-1,1]$ for every $1 \leq j \leq k$.
\end{proof}

\begin{proof}[Proof of Theorem~\ref{thm:advtrainhelps}]
Lemma~\ref{lem:strictconvex} states that $\mathcal{L}_{2}$ is convex and that transitions between cells are strictly convex.  The cases in the theorem statement correspond to the cases in Lemma~\ref{lem:piecewiseconvex} which describe the geometry of the cell containing the origin. Finally Lemma~\ref{lem:optnotinnull} states that any optimal solution must be in the rowspace of $X$ and Lemma~\ref{lem:subdiff} states that $\mathcal{L}_{2}$ is subdifferentiable everywhere.
\end{proof}

\subsection{Proof of Theorem~\ref{thm:minnormsoladv}}

\begin{proof}
The gradient at the minimum $L_{2}$ norm solution $X^{\top}\alpha$ is 
\begin{align*}
    \nabla \mathcal{L}_{2}(X^{\top}\alpha) &= X^{\top}(XX^{\top}\alpha - y) + \epsilon \|X^{\top}\alpha\|_{2} X^{\top}s + \epsilon \|XX^{\top}\alpha - y\|_{1} \frac{X^{\top}\alpha}{\|X^{\top}\alpha\|_{2}} + \epsilon^2 n X^{\top}\alpha\\
    &= \epsilon \|X^{\top}\alpha\|_{2} X^{\top}s  + \epsilon^2 n X^{\top}\alpha.
\end{align*}
Setting the $\nabla \mathcal{L}_{2} = 0$ gives 
\begin{align}
\label{equ:gradzero}
    -X^{\top}s &= \epsilon n \frac{X^{\top}\alpha}{\|X^{\top}\alpha\|_{2}}\\
   -\sum_{i=1}^{n}s_i x_i &= \frac{\epsilon n}{\|X^{\top}\alpha\|_{2}} \left(\sum_{i \in \mathcal{P}}\alpha_{+} x_i + \sum_{j \in \mathcal{N}} \alpha_{-}x_j\right)\nonumber
\end{align}
Lemma~\ref{lem:subdiff} states that, at $X^{\top}\alpha$, there exists a subgradient for every setting of $s \in [-1,1]^n$. To prove the result we must show that there exists some setting of $s$ that satisfies Equation~\ref{equ:gradzero}, which we will do by showing that, under the conditions on $\epsilon$, the coefficient of each $x_i$ on the right hand side of Equation~\ref{equ:gradzero} is in the range $[-1,1]$. 

Since $s_{i} \in [-1,1]$ the negative sign on the left hand side of Equation~\ref{equ:gradzero} is inconsequential. It is sufficient to show that $\frac{\epsilon n \alpha_{+}}{\|X^{\top}\alpha\|_{2}}$ and  $\frac{\epsilon n \alpha_{-}}{\|X^{\top}\alpha\|_{2}}$ are in the range $[-1,1]$. Necessity follows from the fact that the rows of $X$ are linearly independent. 

\begin{align*}
    \frac{\epsilon n \alpha_{+}}{\|X^{\top}\alpha\|_{2}} &= \epsilon \frac{(n_{+} + n_{-})\alpha_{+}}{\sqrt{(n_{+}\alpha_{+} - n_{-}\alpha_{-})^2 + 2(n_{+}\alpha_{+} + n_{-}\alpha_{-})^2 + n_{+}\alpha_{+}^2 + 5n_{-}\alpha_{-}^2}}\\
    &= \epsilon \frac{4n_{-}^2 + 4n_{-}n_{+}+5n_{+}+5n_{-}}{\sqrt{64n_{+}^2n_{-}^2 + 160n_{+}^2n_{-} + 75n_{+}^2 + 32n_{+}n_{-}^2+60n_{+}n_{-}+70n_{+} + 3n_{-}^2 + 5n_{-}}}\\
    &\leq 1 
\end{align*}
where the last inequality follows from the condition on $\epsilon$. Note also that  $\frac{\epsilon n \alpha_{+}}{\|X^{\top}\alpha\|_{2}} \geq 0$ by definition. The case for $\alpha_{-}$ is similar. 

Now assume that $n_{+} = cn_{-}$. The right hand side of Equation~\ref{equ:epscond1} becomes
\begin{equation}
\label{equ:epscond2}
    \frac{\sqrt{64c^2n_{-}^4 + 160c^2n_{-}^3 + 75c^2n_{-}^2 + 32cn_{-}^3+60cn_{-}^2+70cn_{-} + 3n_{-}^2 + 5n_{-}}}{\max\left(4n_{-}^2 + 4cn_{-}^2+5cn_{-}+5n_{-},4c^2n_{-}^2 + 4cn_{-}^2+cn_{-}+n_{-}\right)}.
\end{equation}
The maximum evaluates as 
\begin{equation*}
    \max\left(4n_{-}^2 + 4cn_{-}^2+5cn_{-}+5n_{-},4c^2n_{-}^2 + 4cn_{-}^2+cn_{-}+n_{-}\right) = \begin{cases}
    4n_{-}^2 + 4cn_{-}^2+5cn_{-}+5n_{-} & \text{if } c<\frac{1 + n_{-}}{n_{-}}\\
    4c^2n_{-}^2 + 4cn_{-}^2+cn_{-}+n_{-} & \text{if } c \geq  \frac{1 + n_{-}}{n_{-}}.
    \end{cases}
\end{equation*}
Within each of these ranges it can be checked, using Mathematica, that the gradient of Equation~\ref{equ:epscond2} is negative. Thus we can consider the limit as $n_{-} \rightarrow \infty$, which gives the lower bound
\begin{equation*}
    \frac{\sqrt{64c^2n_{-}^4 + 160c^2n_{-}^3 + 75c^2n_{-}^2 + 32cn_{-}^3+60cn_{-}^2+70cn_{-} + 3n_{-}^2 + 5n_{-}}}{\max\left(4n_{-}^2 + 4cn_{-}^2+5cn_{-}+5n_{-},4c^2n_{-}^2 + 4cn_{-}^2+cn_{-}+n_{-}\right)} \geq \min\left(\frac{2c}{1+c}, \frac{2}{1+c}\right).
\end{equation*}
Taking $\epsilon \leq \min\left(\frac{2c}{1+c}, \frac{2}{1+c}\right)$ is a sufficient condition, but not necessary due to the gap in the lower bound. 
\end{proof}

\section{Corrections}
\label{sec:corrections}
\cite{Wilson17} derive the minimum norm solution using the kernel trick. The optimal solutions $w_{\text{SDG}} = X^{\top}\alpha$ where $\alpha = K^{-1}y$ for $K = XX^{\top}$. They compute  

\begin{equation*}
    K_{ij} = \begin{cases}
        4 & \text{if } i = j \text{ and } y_i = 1\\
        8 & \text{if } i = j \text{ and } y_i = -1\\
        3 & \text{if } i \neq j \text{ and } y_iy_j = 1\\
        1 & \text{if } i \neq j \text{ and } y_iy_j = -1
    \end{cases}
\end{equation*}
and positing, correctly, that $\alpha_i = \alpha_{+}$ if $y_i = 1$ and $\alpha_{i} = \alpha_{-}$ if $y_i = -1$ they derive the system of equations
\begin{align*}
    (3n_{+}+1)\alpha_{+} + n_{-}\alpha_{-} &= 1\\
    n_{+}\alpha_{+} + (3n_{-} + 3)\alpha_{-} &= -1
\end{align*}
which gives 
\begin{equation*}
    \alpha_{+} = \frac{4n_{-} + 3}{9n_{+} + 3n_{-}+8n_{+}n_{-}+5},\quad \alpha_{-} =- \frac{4n_{+} + 1}{9n_{+} + 3n_{-}+8n_{+}n_{-}+5}.
\end{equation*}
\cite{Wilson17} mistakenly dropped the negative in $\alpha_{-}$. Unfortunately there is an additional minor mistake in the linear system. The system is derived by computing 
\begin{align*}
    (K\alpha)_{i} &= \begin{cases}
                    4\alpha_{i} + \sum_{j \in \mathcal{P} - i} 3\alpha_j + \sum_{j \in \mathcal{N}} \alpha_j & \text{if } y_i = 1\\
                    8\alpha_i + \sum_{j \in \mathcal{P}}\alpha_j + 3 \sum_{j \in \mathcal{N}-i}\alpha_j & \text{if } y_j = 1
                \end{cases}\\
                &= \begin{cases}
                    \alpha_{i} + \sum_{j \in \mathcal{P}} 3\alpha_j + \sum_{j \in \mathcal{N}} \alpha_j & \text{if } y_i = 1\\
                    5\alpha_i + \sum_{j \in \mathcal{P}}\alpha_j + 3 \sum_{j \in \mathcal{N}}\alpha_j & \text{if } y_j = 1
                \end{cases}.
\end{align*}
Subtracting equations we reach the conclusion that $\alpha_i = \alpha_{+}$ if $y_i = 1$ and $\alpha_{i} = \alpha_{-}$ if $y_i = -1$. Then it's clear that there are really only two equations in this system 
\begin{align*}
    (3n_{+}+1)\alpha_{+} + n_{-}\alpha_{-} &= 1\\
    n_{+}\alpha_{+} + (3n_{-} + 5)\alpha_{-} &= -1
\end{align*}
which gives 
\begin{equation*}
        \alpha_{+} = \frac{4n_{-} + 5}{15n_{+} + 3n_{-}+8n_{+}n_{-}+5},\quad \alpha_{-} =- \frac{4n_{+} + 1}{15n_{+} + 3n_{-}+8n_{+}n_{-}+5}.
\end{equation*}

\end{document}